\documentclass{article}
\usepackage[hscale=0.7,vscale=0.7]{geometry}
\usepackage[utf8]{inputenc}

\usepackage[square,numbers]{natbib}

\usepackage{color}
\usepackage{verbatim}
\usepackage[affil-it]{authblk}

\usepackage{hyperref}       
\usepackage{url}            
\usepackage{booktabs}       
\usepackage{amsfonts}       
\usepackage{microtype}      
\usepackage{algorithm}
\usepackage{algorithmic}

\usepackage{graphicx}
\usepackage{caption}
\usepackage{subcaption}
\usepackage{xspace}

\usepackage{amsmath}
\usepackage{amsthm}
\usepackage{natbib}
\setcitestyle{authoryear,open={(},close={)}}
\usepackage{bbm}

\usepackage[titletoc,title]{appendix}
\allowdisplaybreaks

\makeatletter
\newtheorem*{rep@theorem}{\rep@title}
\newcommand{\newreptheorem}[2]{%
\newenvironment{rep#1}[1]{%
 \def\rep@title{#2 \ref{##1}}%
 \begin{rep@theorem}}%
 {\end{rep@theorem}}}
\makeatother

\newreptheorem{theorem}{Theorem}
\newreptheorem{definition}{Definition}

\newtheorem{theorem}{Theorem}[section]
\newtheorem{lemma}[theorem]{Lemma}
\newtheorem{proposition}[theorem]{Proposition}

\newtheorem{assumption}[theorem]{Assumption}
\newenvironment{myproof}[1][\proofname]{\proof[#1]\mbox{}}{\endproof}

\title{Risk-Aware Algorithms for Adversarial Contextual Bandits}
\author{Wen Sun$^1$,  Debadeepta Dey$^2$, and Ashish Kapoor$^2$}
\affil{$^1$Robotics Institute, School of Computer Science, \\ Carnegie Mellon University \\ \href{mailto:wensun@cs.cmu.com}{wensun@cs.cmu.com} \\
$^2$Microsoft Research, Redmond  \\
\href{mailto:dedey@microsoft.com}{dedey@microsoft.com}, \href{mailto:akapoor@microsoft.com}{akapoor@microsoft.com}
}

\begin{document}
\maketitle

\begin{abstract}
In this work we consider adversarial contextual bandits with risk constraints. At each round, nature prepares a context, a cost for each arm, and additionally a risk for each arm. The learner leverages the context to pull an arm and then receives the corresponding cost and risk associated with the pulled arm. In addition to minimizing the cumulative cost, the learner also needs to satisfy long-term risk constraints -- the average of the cumulative risk from all pulled arms should not be larger than a pre-defined threshold. To address this problem, we first study the full information setting where in each round the learner receives an adversarial convex loss and a convex constraint. We develop a meta algorithm leveraging online mirror descent for the full information setting and extend it to contextual bandit with risk constraints setting using expert advice. Our algorithms can achieve near-optimal regret in terms of minimizing the total cost, while successfully maintaining a sublinear growth of cumulative risk constraint violation.
\end{abstract}

\section{Introduction}
The \emph{Contextual Bandits} problem \citep{langford2008epoch} has received a large amount of attention in the last decade.  Different from the classic multi-armed bandits problem \citep{auer2002finite,bubeck2012regret}, in contextual bandits, the learner can leverage contextual information to make a decision about which arm to pull. Starting in a completely unknown environment, the learner gradually learns to maximize the cumulative reward by interacting with the environment: in each round, given the contextual information, the learner chooses an arm to pull based on the history of the interaction with the environment, and then receives the reward associated with the pulled arm. For the special case where contexts and rewards are i.i.d sampled from a fixed unknown distribution, there exists an oracle-based computationally efficient algorithm \citep{agarwal2014taming} that achieves near-optimal regret rate. Recently, the authors in \citep{rakhlin2016bistro,syrgkanis2016improved} developed oracle-based computationally efficient algorithms for the hybrid case where the contexts are i.i.d while the rewards could be adversarial, though the regret rate from proposed algorithms are not near-optimal.  For both adversarial contexts and rewards, EXP4 \citep{auer2002nonstochastic} and EXP4.P \citep{beygelzimer2011contextual} are state-of-the-art algorithms, which achieve near-optimal regret rate, but are not computationally efficient.

Recently, a few authors have started to incorporate global constraints into the multi-armed bandit and contextual bandits problem where the goal of the learner is to maximize the reward while satisfying the constraints to some degree. In multi-armed bandit setting, previous work considered special cases such as single resource budget constraint \citep{ding2013multi,madani2004budgeted} and multiple resources budget constraints \citep{badanidiyuru2013bandits}.  Resourceful Contextual Bandits \citep{badanidiyuru2014resourceful} first introduced resource budget constraints to contextual bandits setting. The algorithm proposed in \citep{badanidiyuru2014resourceful} enjoys a near-optimal regret rate but lacks a computationally efficient implementation. Later on, the authors in \citep{agrawal2015contextual} generalize the setting in \citep{badanidiyuru2014resourceful} to contextual bandits with a global convex constraint and a concave objective and propose an oracle-based algorithm built on the ILOVETOCONBANDITS algorithm from \citep{agarwal2014taming}.  Recently  \cite{agrawal2015linear} introduce a UCB style algorithm for linear contextual bandits with knapsack constraints. The settings considered in \citep{badanidiyuru2014resourceful,agrawal2015contextual,agrawal2015linear} mainly focused on the stochastic case where contexts and rewards are i.i.d, and the constraints are pre-fixed before the game starts (i.e., time-independent, non-adversarial). To the best of our knowledge, the work presented in this paper is the first attempt to extend the previous work to the adversarial setting.

 This paper considers contextual bandits with risk constraints, where for each round, the environment prepares a context, a cost for each arm,\footnote{In order to be consistent to classic Online Convex Programming setting, in this work we consider minimizing cost, instead of maximizing reward. } and a risk for each arm. The learner pulls an arm using the contextual information and receives the cost and risk associated with the pulled arm. Given a pre-defined risk threshold, the learner ideally needs to make decisions (i.e., designing a distribution over all arms) such that the average risk is no larger than the threshold in every round, while minimizing the cumulative cost as fast as possible. Such adversarial risk functions are common in many real world applications. For instance, when a robot navigating in an unfamiliar environment,  risk (e.g., probability of being collision, energy consumption, and safety with respect to other robots or even human around the robot) and reward of taking a particular action may dependent on the robot's current state (or the whole sequence of states traversed by the robot so far), while the sequential states visited by the robot are unlikely to be i.i.d or even Markovian.

To address the adversarial contextual bandit with risk constraints problem, we first study the problem of online convex programming (OCP) with constraints in the full information setting, where at each round, the environment prepares a convex loss, and additionally a convex constraint for the learner. The learner wants to minimize its cumulative loss while satisfying the constraints as possible as she could. The online learning with constraints setting is first studied in \citep{mannor2009online} in a two-player game setting. Particularly the authors constructed a two-player game where there exists a strategy for the adversary such that among the strategies of the player that \emph{satisfy the constraints on average},  there is no strategy can achieve no-regret property in terms of maximizing the player's reward.  Later on \citep{mahdavi2012trading,jenatton2015adaptive} considered the online convex programming framework where they  introduced a pre-defined global constraint and designed algorithms that achieve no-regret property on loss functions while maintaining the accumulative constraint violation grows sublinearly.  Though the work in \citep{mahdavi2012trading,jenatton2015adaptive} did not consider time-dependent, adversarial constraints, we find that their online gradient descent (OGD) \citep{Zinkevich2003_ICML} based algorithms are actually general enough to handle adversarial time-dependent constraints. We first present a family of online learning algorithms based on Mirror Descent (OMD) \citep{beck2003mirror,bubeck2015convex}, which we show achieves near-optimal regret rate with respect to loss  and maintains the growth of total constraint violation to be sublinear. With a specific design of a mirror map, our meta algorithm reveals a similar algorithm shown in \citep{mahdavi2012trading}.

The mirror descent based algorithms in the full information online learning setting also enables us to derive a Multiplicative Weight (MW) update procedure using expert advice by choosing negative entropy as the mirror map. Note that MW based update procedure is important when extending to partial information contextual bandit setting. The MW based update procedure can ensure the regret is polylogarithmic in the number of experts , instead of polynomial in the number of experts from using the OGD-based algorithms \citep{mahdavi2012trading,jenatton2015adaptive}. Leveraging the MW update procedure developed from the online learning setting, we present algorithms called EXP4.R (EXP4 with Risk Constraints) and EXP4.P.R (EXP4.P with Risk Constraints). For EXP4.R we show that in expectation, the algorithm can achieve near optimal regret in terms of minimizing cost while ensuring the average of the accumulative risk is no larger than the pre-defined threshold. For EXP4.P.R, we present a high probability statement for regret bound and cumulative risk  bound and introduces a tradeoff parameter that shows how one can trade between the risk violation and the regret of cost.


The rest of the paper is organized as follows. We  introduce necessary definitions and problem setup in Sec.~\ref{sec:pre}. We then deviate to the full information online learning setting where we introduce sequential, adversarial convex constraints in Sec.~\ref{sec:online_learning}. In Sec.~\ref{sec:contextual}, we move to contextual bandits with risk constraints setting to present and analyze the EXP4.R and EXP4.P.R algorithm.

\section{Preliminaries}
\label{sec:pre}
\subsection{Definitions}
For any function $R(x): \mathcal{X}\to \mathcal{R}$, it is strongly convex with respect to some norm $\|\cdot\|$ if and only if there exists a constant $\alpha\in\mathcal{R}^+$  such that:
\begin{align}
R(x) \geq R(x_0) + \nabla R(x_0)^T (x - x_0) + \frac{\alpha}{2}\|x - x_0\|^2. \nonumber
\end{align} Given a strongly convex function $R(\cdot)$, the Bregman divergence $D_R(\cdot,\cdot): \mathcal{X}\times\mathcal{X}\to \mathcal{R}$ is defined as follows:
\begin{align}
D_{R}(x, x') =  R(x) - R(x') - \nabla R(x')^T(x - x'). \nonumber
\end{align}

\subsection{Online Convex Programming with  Constraints}
Under the full information setting, in each round, the learner makes a decision $x_t\in\mathcal{X}\subseteq\mathcal{R}^d$, and then receive a convex loss function $\ell_t(\cdot)$ and a convex constraint  in the form of $f_t(\cdot)\leq 0$. The learner suffers loss $\ell_t(x)$. 
The work in \citep{mahdavi2012trading} considers a similar setting but with a known, pre-defined global constraint. Instead of projecting the decision $x$ back to the convex set induced by the global constraint $f(\cdot)$, \citep{mahdavi2012trading} introduces an algorithm that achieves no-regret on loss while satisfying the global constrain in a long-term perspective.  Since exactly satisfying adversarial constraint in every round is impossible, we also consider constraint satisfaction in a long-term perspective.  Formally, for the sequence of decisions $\{x_t\}_t$ made by the learner, we define $\sum_{t=1}^T f_t(x_t)$ as the cumulative constraint violation and we want to control the growth of the cumulative constraint violation to be sublinear: $\sum_{t=1}^T f_t(x_t)\in o(T)$, so that for the \emph{long-term constraint} $\frac{1}{T}\sum_{t=1}^T f_t(x)$, we have:
\begin{align}
\lim_{T\to\infty}\frac{1}{T}\sum_{t=1}^T f_t(x_t)\leq 0.
\end{align}

Though we consider adversarial constraints, we do place one assumption on the decision set $\mathcal{X}$ and the constraints: we assume that the decision set $\mathcal{X}$ is rich enough such that in hindsight, we have $x\in\mathcal{X}$ that can satisfy the all constraints: $\mathcal{O} \dot{=} \{x\in\mathcal{X}: f_t(x) \leq 0, \forall t\}\neq \emptyset$.  In terms of the definition of regret, we compete with the optimal decision $x^*\in \mathcal{O}$ that minimizes the total loss in hindsight:
\begin{align}
\label{eq:constraint_on_best}
x^* = \arg\min_{x\sim\mathcal{O}} \sum_{t=1}^T \ell_t(x).
\end{align} Though one probably would be interested in  competing against the best decision from the set of decisions that satisfy the constraints in average: $\mathcal{O}' \dot{=} \{x\in\mathcal{X}: (1/T)\sum_{t}^T f_t(x) \leq 0\}$,  in general it is impossible to complete agains the best decision in $\mathcal{O}'$ in hindsight. The following proposition adapts the discrete 2-player game from proposition 4 in \citep{mannor2009online} for the online convex programming with adversary constraints setting and shows  the learner is impossible to compete agains $\mathcal{O}'$:
\begin{proposition}
\label{prop:wrong_decision_set}
There exist a decision set $\mathcal{X}$, a sequence of convex loss functions $\{\ell_t(x)\}$, and a sequence of convex constraints $\{f_t(x)\leq0\}$, such that for any sequence of decisions $\{x_1, ..., x_t, ...\}$, if it satisfies the long-term constrain as $\limsup_{t\to\infty} \frac{1}{t}\sum_{i=1}^t f_i(x_i) \leq 0$, then if competing against $\mathcal{O}'$, the  regret grows at least linearly:
\begin{align}
\limsup_{t\to\infty}\big( \sum_{i=1}^t \ell_i(x_i)  -  \min_{x\in\mathcal{O}'}\sum_{i=1}^t \ell_i(x)\big) = \Omega(t).
\end{align}
\end{proposition} The proof of the proposition can be find in Sec.~\ref{sec:prop_wrong_decision_set} in Appendix.  Hence in the rest of the paper, we have to restrict to $\mathcal{O}$, which is a subset of $\mathcal{O}'$.  The regret of loss $R_{\ell}$ and the cumulative constraint violation $R_{f}$ are defined as:
\begin{align}
R_{\ell} = \sum_{t=1}^T \ell_t(x_t) - \sum_{t=1}^{T} \ell_t(x^*), \;\;\;\;\;R_{f}  = \sum_{t=1}^T f_t(x_t). \nonumber
\end{align} We want both $R_l, R_f \in o(T)$. We will assume the decision set is bounded as $\max_{x_1,x_2\in\mathcal{X}}D_R(x_1,x_2)\leq B\in\mathcal{R}^+$, $x\in\mathcal{X}$ is bounded as $\|x\|\leq X\in\mathcal{R}^+$, the loss function is bounded as $|\ell_t(\cdot)|\leq F\in\mathcal{R}^+$, the constraint is bounded $|f_t(\cdot)|\leq D\in\mathcal{R}^+$ and the gradient of the loss and constraint is also bounded as $\max\{\|\nabla_x \ell_t(x)\|_{*},\|\nabla_x f_t(x)\|_{*}\}\leq G\in\mathcal{R}^+$, where $\|\cdot\|_*$ is the dual norm with respect to $\|\cdot\|$ defined for $\mathcal{X}$.

The setting with a global constraint considered in \citep{mahdavi2012trading} is a special case of our setting. Set $f_t = f$, where $f$ is the global constraint. If $R_f\in o(T)$, by Jensen's inequality, we have $f(\sum_{t=1}^T x_t / T) \leq \sum_{t=1}^T f(x_t) / T = o(T)/T \to 0$, as $T\to \infty$.

\subsection{Contextual Bandits with Risk Constraints}
For contextual bandits with risk constraints, let $[K]$ be a finite set of $K$ arms, $\mathcal{S}$ be the space of contexts. Except for providing context and the cost for each action, the environment will also provide the \emph{risk} for each action (e.g., how dangerous or risky it would be by taking the action under the current context). More formally, at every time step $t$, the environment generates a context $s_t\in\mathcal{S}$, a K-dimensional cost vector $c_t\in [0,1]^K$, 
and a risk vector $r_t\in [0,1]^K$. The environment then reveals the context $s_t$ to the learner, and the learner then propose a probability distribution $p_t\in\Delta([K])$ over all arms. Finally the learner samples an action $a_t\in[K]$ according to $p_t$ and receives the cost and risk associated to the chosen action: $c_t[a_t]$ and $r_t[a_t]$ (we denote $c[i]$ as the $i$'th element of vector $c$). The learner ideally want to make a sequence of decisions that has low accumulative cost and also satisfies the constraint that related to the risk: $p_t^Tr_t\leq \beta$ where $\beta\in[0,1]$ is a pre-defined threshold. 


We address this problem by leveraging experts' advice. Given the expert set $\Pi$ that consists of $N$ experts $\{\pi_i\}_{i=1}^N$, where each expert $\pi\in\Pi: \mathcal{S}\to \Delta([K])$, gives advice by mapping from the context $s$ to a probability distribution $p$ over arms. The learner then properly combines the experts' advice $\{\pi_i(s)\}_{i=1}^N$ (e.g., compute the average $\sum_{i=1}^N \pi_i(s)/N$) to generate a distribution over all arms. With risk constraints, distributions over policies in $\Pi$ could be strictly more powerful than any policy in $\Pi$ itself. We aim to compete against this more powerful set, which is a stronger guarantee than simply competing with any fixed policy in $\Pi$. Given any distribution $w\in\Delta(\Pi)$, the mixed policy resulting from $w$ can be regarded as: sample policy $i$ according to $w$ and then sample an arm according to $\pi_i(s)$, given any context $s$.
Though we do not place any statistical assumptions (e.g., i.i.d) on the sequence of cost vectors $\{c_t\}$ and risk vectors $\{r_t\}$, we assume the policy set $\Pi$ is rich enough to satisfy the following assumption:
\begin{assumption}
The set of distributions from $\Delta(\Pi)$ whose mixed policies satisfying the risk constraints in expectation is non-empty:
\begin{align}
\mathcal{P} \dot{=} \{w\in\Delta(\Pi): \mathbb{E}_{i\sim w,j\sim \pi_i(s_t)} r_t[j] \leq \beta, \forall t\} \neq \emptyset. \nonumber
\end{align}
\end{assumption} Namely we assume that the distribution set $\Delta(\Pi)$ is rich enough such that there always exists at least one mixed policy that can satisfy all risk constraints in hindsight. Similar to the full information setting, competing against the set of mixed policies that satisfy the constraint on average, namely $\mathcal{P}'=\{w\in\Delta(\Pi): \sum_{t=1}^T \mathbb{E}_{i\sim w,j\sim \pi_i(s_t)} r_t[j] /T\leq \beta\}$,  is impossible in the partial information setting.\footnote{Otherwise we can just directly apply the algorithm designed for the partial information setting to the full information setting}  Hence we define the best mixed policy in hindsight as:
\begin{align}
w^* =\arg\min_{w\sim \mathcal{P}} \sum_{t=1}^T \mathbb{E}_{i\sim w,j\sim \pi_i(s_t)} c_t[j].
\end{align}
Given any sequence of decisions $\{a_t\}_{t=1}^T$ generated from some algorithm, let us define the average pseudo-regret in expectation as:
\begin{align}
\bar{R}_c = \frac{1}{T}\mathbb{E}\big[\sum_{t=1}^T c_t[a_t] - \sum_{t=1}^T\mathbb{E}_{i\sim w^*, j\sim \pi_i(s_t)}c_t[j]\big], \nonumber
\end{align} where the expectation is taken with respect to the randomness of the algorithm. The expected cumulative risk constraint violation as:
\begin{align}
\bar{R}_{r} = \frac{1}{T} \mathbb{E}\big[ \sum_{t=1}^T (r_t[a_t] - \beta)\big]. \nonumber
\end{align} The goal is to achieve near-optimal regret rate for $\bar{R}_c$ (i.e., $\bar{R}_c = O(\sqrt{TK\ln(|\Pi|)})$) while maintaining $\bar{R}_r$ growing sublinearly.  Without loss of generality, we also assume that the policy class $\Pi$ contains a policy that always outputs uniform distribution over arms (i.e., assign probability $1/K$ to each arm).

Similarly, we define regret and constraint violation \emph{without} expectation as:
\begin{align}
R_c = \frac{1}{T}& \big[\sum_{t=1}^T c_t[a_t] - \sum_{t=1}^T\mathbb{E}_{i\sim w^*, j\sim \pi_i(s_t)}c_t[j]\big], \;\;\;\;  {R}_{r} = \frac{1}{T}\big[ \sum_{t=1}^T (r_t[a_t] - \beta)\big]. \nonumber
\end{align} The goal is to minimize $R_c$ and $R_r$ in high probability.

\section{Online Learning with Constraints}
\label{sec:online_learning}
The online learning with adversarial constraints setting is similar to the one considered in \citep{mannor2009online,jenatton2015adaptive} except that they only have a pre-defined fixed global constraint. However we find that their algorithms and analysis are general enough to extend to the online learning with adversarial sequential constraints.  In \citep{mannor2009online,jenatton2015adaptive}, the algorithms introduce a Lagrangian dual parameter and perform  online gradient descent on $x$ and online gradient ascent on the dual parameter. Since in this work we are eventually interested in reducing the contextual bandit problem to the full information online learning setting, simply adopting the OGD-based approaches from \citep{mannor2009online,jenatton2015adaptive} will not give a near optimal regret bound. Hence, developing the corresponding Multiplicative Weight (MW) update procedure is essential for a successful reduction from adversarial contextual bandit to full information online learning setting.

\subsection{Algorithm}
We use the same saddle-point convex concave formation from \citep{mannor2009online,jenatton2015adaptive} to design a composite loss function as:
\begin{align}
\mathcal{L}_t(x, \lambda) = \ell_t(x) + \lambda f_t(x) - \frac{\delta\mu}{2}\lambda^2,
\end{align} where $\delta\in\mathcal{R}^+$.
Alg.~\ref{alg:ol_constraints} leverages online mirror descent (OMD) for updating the $x$ (Line~\ref{line:mirror_descent_1} and Line~\ref{line:mirror_descent_2}) and online gradient ascent algorithm for updating $\lambda$ (Line~\ref{line:update_of_lambda}). 
Note that if we replace the regularization function $R(x)$ by $\|x\|_2^2$, we reveals gradient descent based update rule that is similar to the one in \citep{mahdavi2012trading}. 

\begin{algorithm}[tb]
\caption{OCP with Constraints via OMD}
 \label{alg:ol_constraints}
\begin{algorithmic}[1]
  \STATE {\bfseries Input:} Parameters $\mu,\delta$, learning rate $\mu$, mirror map $R$.
  \STATE Initialize $x_0 \in\mathcal{X}$ and $\lambda_0 = 0$.
 \FOR {t = 0 to T}
    \STATE Learner proposes $x_t$.
    \STATE Receive loss function $\ell_t$ and constraint $f_t$.
    \STATE Set $\tilde{x}_{t+1}$ such that $\nabla R(\tilde{x}_{t+1}) = \nabla R({x}_t) - \mu\nabla_x\mathcal{L}(x_t,\lambda_t)$. \label{line:mirror_descent_1}
    \STATE Projection: $x_{t+1} = \arg\min_{x\in\mathcal{X}}D_R(x, \tilde{x}_{t+1})$. \label{line:mirror_descent_2}
    \STATE Update $\lambda_{t+1} = \max\{0, \lambda_t + \mu \nabla_{\lambda}\mathcal{L}(x_t,\lambda_t))\}$. \label{line:update_of_lambda}
 \ENDFOR
\end{algorithmic}
\end{algorithm}

\subsection{Analysis of Alg.~\ref{alg:ol_constraints}}
Throughout our analysis, we assume the regularization function $R(x)$ is $\alpha$-strongly convex. For simplicity, we assume the number of rounds $T$ is given and we consider the asymptotic property of Alg.~\ref{alg:ol_constraints} when $T$ is large enough. 

The algorithm should be really understood as running two no-regret procedures: (1) Online Mirror Descent on the sequence of loss $\{\mathcal{L}(x,\lambda_t)\}_t$ with respect to $x$ and (2) Online Gradient Ascent on the sequence of loss $\{\mathcal{L}(x_t,\lambda)\}_t$ with respect to $\lambda$.  Instead of digging into the details of Online Mirror Descent and Online Gradient ascent, our analysis simply leverage the existing analysis of online mirror descent and online gradient ascent and show how to combine them to derive the regret bound and constraint violation bound for Alg.~\ref{alg:ol_constraints}. 

\begin{theorem}
\label{them:ol}
Let $R(\cdot)$ be a $\alpha$-strongly convex function. Set $\mu = \sqrt{\frac{B}{T(D^2+G^2/\alpha)}}$ and $\delta = \frac{2G^2}{\alpha}$. For any convex loss $\ell_t(x)$, convex constraint $f_t(x)\leq 0$, under the assumption that $\mathcal{O}\neq \emptyset$, the family of algorithms induced by Alg.~\ref{alg:ol_constraints} have the following property:
\begin{align}
R_{\ell} / T \leq O(1/\sqrt{T}), \;\;\;\; R_{f}/T \leq O(T^{-1/4}). \nonumber
\end{align}
\end{theorem}

\begin{myproof}[Proof Sketch of Theorem~\ref{them:ol}]
Since the algorithm runs online mirror descent on the sequence of loss $\{\mathcal{L}_t(x,\lambda_t)\}_t$ with respect to $x$, using the existing results of online mirror descent (e.g., Theorem 4.2 and Eq. 4.10 from \cite{bubeck2015convex}), we know that for the sequence of $\{x_t\}_t$:
\begin{align}
\label{eq:mirror_descent_analysis_1}
&\sum_{t=0}^T(\mathcal{L}_t(x_t, \lambda_t) - \mathcal{L}_t(x, \lambda_t))\leq \frac{D_R(x, x_0)}{\mu} + \frac{\mu}{2\alpha}\sum_{t=0}^T\|\nabla_{x}\mathcal{L}(x_t,\lambda_t)\|_{*}^2.
\end{align}
Also, we know that the algorithm runs online gradient ascent on the sequence of loss $\{\mathcal{L}_t(x_t, \lambda)\}_t$ with respect to $\lambda$, using the existing analysis of online gradient descent \citep{Zinkevich2003_ICML}, we have for the sequence of $\lambda_t$:
\begin{align}
\label{eq:ogd_analysis_1}
&\sum_{t=0}^T\mathcal{L}_t(x_t,\lambda) - \sum_{t=1}^T\mathcal{L}_t(x_t,\lambda_t) \leq \frac{1}{\mu}\lambda^2 + \frac{\mu}{2}\sum_{t=1}^T \big(\frac{\partial \mathcal{L}_t(w_t,\lambda_t)}{\partial \lambda_t}\big)^2,
\end{align}

Note that for $(\partial\mathcal{L}_t(x_t,\lambda_t)/\partial \lambda_t)^2 = (f_t(x_t) - \delta\mu\lambda_t)^2\leq 2f_t^2(x_t) + 2\delta^2\mu^2\lambda_t^2 \leq 2D^2 + \delta^2\mu^2\lambda_t^2$. Similarly for $\|\nabla_x\mathcal{L}_t(x_t,\lambda_t)\|_*^2$, we also have:
\begin{align}
&\|\nabla_x\mathcal{L}_t(x_t,\lambda_t)\|_*^2 \leq 2\|\nabla\ell_t(x_t)\|_*^2 + 2\|\lambda_t\nabla f_t(x_t)\|_*^2 \leq 2G^2(1+\lambda_t^2),
\end{align} where we first used triangle inequality for $\|\nabla_x\mathcal{L}_t(x_t,\lambda_t)\|_*$ and then use the inequality of $2ab\leq a^2+b^2, \forall a,b\in \mathcal{R}^+$. Note that we also assumed that the norm of the gradients are bounded as $max(\|\nabla\ell_t(x_t)\|_*, \|\nabla f_t(x_t)\|_*) \leq G\in\mathcal{R}^+$.
Now sum Inequality~\ref{eq:mirror_descent_analysis_1} and \ref{eq:ogd_analysis_1} together, we get:
\begin{align}
&\sum_{t}\mathcal{L}_t(x_t,\lambda) - \mathcal{L}_t(x,\lambda_t)\nonumber\\
&\leq\frac{2D_R(x,x_0) + \lambda^2}{2\mu} + \sum_{t}\mu(D^2 +  \delta^2\mu^2\lambda_t^2)+ \sum_t\frac{\mu G^2}{\alpha}(1+\lambda_t^2) \nonumber\\
& = \frac{2D_R(x,x_0) + \lambda^2}{2\mu} + T\mu(D^2 + \frac{ G^2}{\alpha})  + \mu(\delta^2\mu^2 + \frac{ G^2}{\alpha})\sum\lambda_t^2.
\end{align}

Substitute the form of $\mathcal{L}_t$ into the above inequality, we have:
\begin{align}
\label{eq:first_eq_in_ol}
& \sum_t(\ell_t(x_t) - \ell_t(x)) + \sum_t(\lambda f_t(x_t) - \lambda_t f_t(x))+ \frac{\delta\mu}{2}\sum_t\lambda_t^2 - \frac{\delta\mu T}{2}\lambda^2\nonumber\\
& \leq \frac{2D_R(x,x_0) + \lambda^2}{2\mu}  + T\mu(D^2 + \frac{ G^2}{\alpha}) + \mu(\delta^2\mu^2 + \frac{ G^2}{\alpha})\sum_t\lambda_t^2.
\end{align}  Note that from our setting of $\mu$ and $\delta$ we can verify that $\delta \geq \delta^2\mu^2 + G^2/\alpha$,\footnote{For simplicity we assumed $T$ is large enough to be larger than any given constant.} we can remove the term $\sum_{t}\lambda_t^2$ in the above inequality.

Without the term $\sum_{t}\lambda_t^2$, to upper bound the regret on loss $\ell_t$, let us set $\lambda = 0$ and $x = x^*$, we get:
\begin{align}
&\sum_t(\ell_t(x_t) - \ell_t(x^*)) \leq \frac{2D_R(x,x_0)}{2\mu} + T\mu(D^2 + G^2/\alpha) \nonumber\\
&\leq 2\sqrt{D_R(x,x_0)T(D^2+G^2/\alpha)} = O(\sqrt{T}), \nonumber
\end{align} with $\mu = \sqrt{D_R(x,x0)/(T(D^2+G^2/\alpha))}$.
To upper bound $\sum_t f_t(x_t)$, we first observe that we can lower found $\sum_{t=1}^T \ell_t(x_t) - \min_{x}\sum_{t=1}^T \ell_t(x) \geq -2FT$, where $F$ is the upper bound of $\ell(\cdot)$. Replace $\sum_t \ell_t(x_t) - \ell_t(x)$ by $-2FT$ in Eq.~\ref{eq:first_eq_in_ol}, and set $\lambda = (\sum_t f_t(x_t))/(\delta\mu T + 1/\mu)$ (here we assume $\sum_t f_t(x_t) \geq 0$, otherwise we prove the theorem), we can show that:
\begin{align}
&(\sum_{t=1}^T f_t(x_t) )^2\leq \frac{8G^2}{\alpha}D_R(x,x_0) + 2(D^2+\frac{G^2}{\alpha})T  + T^{3/2}\sqrt{8F^2G^2/\alpha}
\end{align} The  RHS of the above inequality is dominated by the term $T^{3/2}\sqrt{8F^2G^2/\alpha}$ when $T$ approaches to infinity. Hence, it is straightforward to show that $\sum_{t=1}^T f_t(x_t) = O(T^{3/4})$.
\end{myproof}

As we can see that if we replace $R(x)$ with $\|x\|_2^2$ in Alg.~\ref{alg:ol_constraints}, we reveal a gradient descent based update procedure that is almost identical to the one in  \citep{mahdavi2012trading}. When $x$ is restricted to a simplex, to derive the multiplicative weight update procedure, we replace $R(x)$ with the negative entropy regularization $\sum_{i} x[i]\ln(x[i])$ and we can achieve the following update steps for $x$:
\begin{align}
x_{t+1}[i] = \frac{x_{t}[i]\exp(-\mu \nabla_x\mathcal{L}_t(x_t,\lambda_t)[i])}{\sum_{j=1}^d x_{t}[j]\exp(-\mu \nabla_x\mathcal{L}_t(x_t,\lambda_t)[j])}. \nonumber
\end{align} We refer readers to \citep{Shwartz2011_FTML,bubeck2015convex} for the derivation of the above equation.

\section{Contextual Bandits with Risk Constraints}
\label{sec:contextual}
When contexts, costs and risks are i.i.d sampled from some unknown distribution, then our problem setting can be regarded as a special case of the setting of contextual bandit with global objective and constraint (CBwRC) considered in \citep{agrawal2015contextual}. In \citep{agrawal2015contextual}, the algorithm also leverages Lagrangian dual variable. The difference is that in i.i.d setting the dual parameter is fixed with respect to the underlying distribution and hence it is possible to estimate the dual variable. 
For instance one can uniformly pull arms with a fixed number of rounds at the beginning to gather information for estimating the dual variable and then use the estimated dual variable for all remaining rounds. However in the adversarial setting, this nice trick will fail since the costs and risks are possibly sampled from a changing distribution. We have to rely on OCP algorithms to keep updating the dual variable to adapt to adversarial risks and costs.

\subsection{Algorithm}
\begin{algorithm}[tb]
\caption{EXP4 with Risk Constraints (EXP4.R)}
 \label{alg:exp4_constraints}
\begin{algorithmic}[1]
  \STATE {\bfseries Input:} Policy set $\Pi$.
  \STATE Initialize $w_0 = [1/N,...,1/N]^T$ and $\lambda_0 = 0$.
 \FOR {t = 0 to T}
    \STATE Receive context $s_t$.
    \STATE Query experts to get advice $\pi_i(s_t), \forall i\in [N]$.
    \STATE Set $p_t = \sum_{i=1}^N w_t[i]\pi_i(s_t)$.
    \STATE Draw action $a_t$ randomly from distribution $p_t$.
    \STATE Receive cost $c_t[a_t]$ and risk $r_t[a_t]$.
    \STATE Set the cost vector $\hat{c}_t\in R^K$ and the risk vector $\hat{r}_t\in R^K$ as follows: for all $i\in [K]$
    \begin{align}
    \hat{c}_t[i] = \frac{c_t[i]\mathbbm{1}(a_t = i)}{p_t[i]}, \;\;\; \hat{r}_t[i] = \frac{r_t[i]\mathbbm{1}(a_t = i)}{p_t[i]}. \nonumber
    \end{align}

    \STATE For each expert $j\in [N]$, set:
    \begin{align}
    \hat{y}_t[j] = \pi_j(s_t)^T\hat{c}_t,\;\;\; \hat{z}_t[j] = \pi_j(s_t)^T\hat{r}_t. \nonumber
    \end{align}
    \STATE Compute $w_{t+1}$, for $i\in [|\Pi|]$: \label{line:update_w}
    	\begin{align}
	w_{t+1}[i] = \frac{w_t[i]\exp\big(-\mu (\hat{y}_t[i] + \lambda_t\hat{z}_t[i])\big)} {\sum_{j=1}^{|\Pi|}  w_t[j] \exp\big(-\mu(\hat{y}_t[j] + \lambda_t\hat{z}_t[j])\big)}. \nonumber
	\end{align}
    \STATE Compute $\lambda_{t+1}$: \label{line:update_lambda}
    \begin{align}
    \lambda_{t+1} = \max\{0, \lambda_t + \mu (w_t^T\hat{z}_t - \beta - \delta\mu\lambda_t)\}. \nonumber
    \end{align}
     \ENDFOR
\end{algorithmic}
\end{algorithm}

Our algorithm EXP4.R (EXP4 with Risk constraints) (Alg.~\ref{alg:exp4_constraints}) extends the EXP4 algorithm to carefully incorporating the risk constraints for updating the probability distribution $w$ over all policies. At each round, it first uses the common trick of importance sampling to form an unbiased estimates of cost vector $\hat{c}$ and risk vector $\hat{r}$. Then the algorithm uses the unbiased estimates of cost vector and risk vector to form unbiased estimates of the cost $\hat{y}[i]$ and risk $\hat{z}[i]$ for each expert $i$. EXP4.R then starts behaving different than EXP4.  EXP4.R introduces a dual variable $\lambda$ and combine the cost and risk together as $\mathcal{L}_t(w,\lambda) = w^T \hat{y}_t + \lambda (w^T\hat{z}_t - \beta) - \frac{\delta\mu}{2}\lambda^2$. We then use Alg.~\ref{alg:ol_constraints} with the negative entropy regularization as a black box online learner to update the weight $w$ and the dual variable $\lambda$ as shown in Lines~\ref{line:update_w} and \ref{line:update_lambda} of Alg.~\ref{alg:exp4_constraints}.

The proposed algorithm EXP4.R is computationally inefficient since similar to EXP4.P, it needs to maintain a probability distribution over the policy set. Though there exist computationally efficient algorithms for stochastic contextual bandits and hybrid contextual bandits, we are not aware of any computationally efficient algorithm for adversarial contextual bandits, even without risk constraints.

\subsection{Analysis of EXP4.R}
We provide a reduction based analysis for EXP4.R by first reducing EXP4.R to Alg.~\ref{alg:ol_constraints} with negative entropic regularization.  For the following analysis, let us define $y_t[j] = \pi_j(s_t)^T c_t$ and $z_t[j] = \pi_j(s_t)^T r_t$, which stand for the expected cost and risk for policy $j$ at round $t$.

Let us define $\mathcal{L}_t(w, \lambda) = w^T\hat{y}_t +\lambda (w^T\hat{z}_t - \beta)- \frac{\delta\mu}{2}\lambda^2$. The multiplicative weight update in Line~\ref{line:update_w} can be regarded as running Weighted Majority on the sequence of loss $\{\mathcal{L}_t(w, \lambda_t)\}_t$, while the update rule for $\lambda$ in Line~\ref{line:update_lambda} can be regarded as running Online Gradient Ascent on the sequence of loss $\{\mathcal{L}_t(w_t, \lambda)\}_t$. Directly applying the classic analysis of Weighted Majority \citep{Shwartz2011_FTML} on the generated sequence of weights $\{w_t\}_t$ and the classic analysis of OGD \citep{Zinkevich2003_ICML} on the generated sequence of dual variables $\{\lambda_t\}_t$, we get the following lemma:

\begin{lemma}
With the negative entropy as the regularization function for $R$, running Alg.~\ref{alg:ol_constraints} on the sequence of linear loss functions $\ell_t(w) = w^T\hat{y}_t$ and linear constraint $f_t(w) = w^T\hat{z}_t - \beta\leq 0$, we have:
\begin{align}
\sum_{t=1}^T\mathcal{L}_t(w_t,\lambda) - \sum_{t=1}^T\mathcal{L}_t(w,\lambda_t)  \leq & \frac{\lambda^2}{\mu} + \frac{\ln (|\Pi|)}{\mu} + \frac{\mu}{2}\sum_{t=1}^T\Big(
\big(\sum_{i=1}^{|\Pi|}w_t[i](2\hat{y}_t[i]^2+2\lambda_t^2\hat{z}_t[i]^2)\big) \nonumber \\
&\;\;\;\;\;\; + (w_t^T\hat{z}_t - \beta - \delta\mu\lambda_t)^2
\Big).
\label{eq:before_expectation_tmp}
\end{align}
\end{lemma}
We defer the proof of the above lemma to Appendix. 
The EXP4.R algorithm has the following property:
\begin{theorem}
\label{them:bandits}
Set $\mu = \sqrt{\ln(|\Pi|)/(T(K+4))}$ and $\delta = 3K$. Assume $\mathcal{P}\neq \emptyset$. EXP4.R has the following property:
\begin{align}
&\bar{R}_c  = O(\sqrt{K\ln(|\Pi|)/T}), \nonumber\\
 &\bar{R}_r = O(T^{-1/4}(K\ln(|\Pi|))^{1/4}). \nonumber
\end{align}
\end{theorem}
\begin{myproof}[Proof Sketch of Theorem~\ref{them:bandits}]
The proof consists of a  combination of the analysis of EXP4 and the analysis of Theorem~\ref{them:ol}. We defer the full proof in Appendix~\ref{sec:exp4_R_analysis}.  We first present several known facts. First, we have $w_t^T\hat{z}_t =  r_t[a_t] \leq 1$ and $w_t^T\hat{y}_t = c_t[a_t]\leq 1$.

For $\mathbb{E}_{a_t\sim p_t}(w_t^T\hat{z}_t - \beta)^2$, we can show that:
$\mathbb{E}_{a_t\sim p_t}(w_t^T\hat{z}_t - \beta)^2 \leq 2+2\beta^2 \leq 4$.

It is also straightforward to show that $\mathbb{E}_{a_t\sim p_t} \hat{y}_t = y_t$ and $\mathbb{E}_{a_t\sim p_t}\hat{z}_t = z_t$.
It is also true that $\mathbb{E}_{a_t\sim p_t}\sum_{i=1}^{|\Pi|}w_t[i]\hat{y}_t[i]^2  \leq K$ and $\mathbb{E}_{a_t\sim p_t}\sum_{i=1}^{|\Pi|}w_t[i]\hat{z}_t[i]^2 \leq K$.

Now  take expectation with respect to the sequence of decisions $\{a_t\}_t$ on LHS of Inequality~\ref{eq:before_expectation_tmp}: 
\begin{align}
\label{eq:exp4_fact_5}
&\mathbb{E}_{\{a_t\}_t}\sum_{t=1}^T\Big[ \mathcal{L}_t(w_t,\lambda) -  \mathcal{L}_t(w,\lambda_t) \Big] \nonumber\\
& = \sum_{t=1}^T\big[\mathbb{E}c_t[a_t] + \lambda(\mathbb{E}r_t[a_t] - \beta)  - y_t^Tw - \lambda_t(z_t^Tw - \beta) + \frac{\delta\mu}{2}\lambda_t^2\big] -  \frac{\delta\mu T}{2}\lambda^2
\end{align}
Now take the expectation with respect to $a_1,...,a_T$ on the RHS of inequality~\ref{eq:before_expectation_tmp}, we can get:
\begin{align}
\label{eq:exp4_fact_6}
&\mathbb{E}[\textit{RHS of Inequality~\ref{eq:before_expectation_tmp}}]
\leq \frac{\lambda^2}{\mu} + \frac{\ln|\Pi|}{\mu} + \mu T(K+4) + \mu(K+\delta^2\mu^2)\sum_{t=1}^T\lambda_t^2.
\end{align} Now we  can chain Eq.~\ref{eq:exp4_fact_5} and \ref{eq:exp4_fact_6} together and  use the same technique that we used in the analysis of Theorem~\ref{them:ol}. Chain Eq.~\ref{eq:exp4_fact_5} and \ref{eq:exp4_fact_6} together and set $w$ to $w^*$ and $\lambda = 0$, it is not hard to show that:
\begin{align}
\label{eq:fact_8}
&\mathbb{E}\big(\sum_{t=1}^T c_t[a_t] -\sum_{t=1}^T y_t^Tw^*\big)  \leq 2\sqrt{\ln(|\Pi|) T(K+4)} = O(\sqrt{TK\ln(|\Pi|)}),  \nonumber
\end{align} where $\mu = \sqrt{\ln(|\Pi|)/(T(K+4))}$.  Set $\lambda =  (\sum_t (\mathbb{E}r_t[a_t] - \beta)) / (\delta\mu T + 2/\mu)$, we can get:
\begin{align}
&(\sum_{t=1}^T(\mathbb{E}r_t[a_t] - \beta))^2 \leq (2\delta\mu T + 4/\mu)\Big( 2T  + 2\sqrt{\ln(|\Pi|) T (K+4)}  \Big)
\end{align} Substitute $\mu = \sqrt{\ln(|\Pi|)/(T(K+4))}$ back to the above equation, it is easy to verity that:
\begin{align}
\label{eq:fact_9}
(\sum_{t=1}^T(\mathbb{E} r_t[a_t] - \beta))^2 \leq O\big(T^{3/2} (K\ln(|\Pi|))^{1/2}\big).
\end{align}
 \end{myproof}

\subsection{Extension To High-Probability Bounds}
The regret bound and constraint violation bound of EXP4.R hold in expectation. In this section, we present an algorithm named EXP4.P.R, which achieves high-probability regret bound and constraint violation bound. The algorithm EXP4.P.R, as indicated by its name, is built on the well-known EXP4.P algorithm \citep{beygelzimer2011contextual}. In this section for the convenience of analysis, without loss of generality, we are going to assume that for any cost vector $c$ and risk vector $r$, we have $c[i]\in[-1,0]$, $r[i]\leq [-1,0], \forall i\in [K]$, 
and $\beta \in [-1,0]$.

The whole framework of the algorithm is similar to the one of EXP4.R, with \emph{only one} modification. For notation simplicity, let us define $\tilde{x}_t[i] = \hat{y}_t[i] + \lambda_t\hat{z}_t[i]$. Note that $\tilde{x}_t[i]$ is an unbiased estimate of $y_t[i] + \lambda_t z_t[i]$. EXP4.P.R modifies EXP4.R by replacing the update procedure for $w_{t+1}$ in Line~\ref{line:update_w} in Alg.~\ref{alg:exp4_constraints} with the following update step:
\begin{align}
w_{t+1}[i] = \frac{w_t[i]\exp(-\mu(\tilde{x}_t[i]- \kappa\sum_{k=1}^K \frac{\pi_i(s_t)[k]}{p_t[k]}))}{\sum_{j=1}^{|\Pi|}w_t[j]\exp(-\mu( \tilde{x}_t[j]- \kappa\sum_{k=1}^K \frac{\pi_j(s_t)[k]}{p_t[k]}))}, \nonumber
\end{align} where $\kappa$ is a constant that will be defined in the analysis of EXP4.P.R. We refer readers to Appendix~\ref{sec:exp4_p_r} for the full version of EXP4.P.R. Essentially, similar to EXP3.P and EXP4.P, we add an extra term $-\kappa \sum_{k=1}^K \frac{\pi_j(s_t)[k]}{p_t[k]}$ to $\tilde{x}_t[i]$. Though $\tilde{x}_t[i] -\kappa \sum_{k=1}^K \frac{\pi_j(s_t)[k]}{p_t[k]}$ is not an unbiased estimation of $y_t[i]+\lambda_t z_t[i]$ anymore, as shown in Lemma~\ref{lemma:high_prob} in Appendix~\ref{sec:analysis_exp4_P_R}, it enables us to upper bound $\sum_t \tilde{x}_t[i] -\kappa \sum_{k=1}^K \frac{\pi_j(s_t)[k]}{p_t[k]}$ using $\sum_t y_t[i] + \lambda_t z_t[i]$ with high probability.

We show that EXP4.P.R has the following performance guarantees:
\begin{theorem}
\label{them:exp4_P_R}
Assume $\mathcal{P}\neq \emptyset$. For any $\epsilon\in (0, 1/2)$, $\nu\in(0,1)$, set $\mu = \sqrt{\frac{\ln(|\Pi|)}{(3K+4)T}}$, $\kappa = \sqrt{\frac{(1+T^{\epsilon})\ln(|\Pi|/\nu)}{TK}}$, and $\delta = T^{-\epsilon+1/2}K$, with probability at least $1-\nu$, EXP4.P.R has the following property:
\begin{align}
& R_c =  O(\sqrt{T^{\epsilon - 1}K\ln(|\Pi|/\nu)}), \nonumber \\
& R_r =  O(T^{-\epsilon/2}\sqrt{K\ln(|\Pi|)}).
\end{align}
\end{theorem} The above theorem introduces a trade-off between the regret of cost and the constraint violation. As $\epsilon\to 0$, we can see that the regret of cost approaches to the  near-optimal one $\sqrt{TK\ln(|\Pi|)}$, but the average risk constrain violation approaches to a constant. Based on specific applications, one may set a specific $\epsilon \in (0,0.5)$ to balance the regret and the constraint violation. For instance, for $\epsilon = 1/3$, one can show that the cumulative regret is $O(T^{2/3}\sqrt{K\ln(|\Pi|)})$ and the average constraint violation is $\tilde{O}(T^{-1/6})$. Note that if one simply runs EXP4.R proposed in the previous section, it is impossible to achieve the regret rate $O(T^{2/3}\sqrt{\ln(|\Pi|)})$ in a high probability statement. As shown in \citep{auer2002nonstochastic},  for EXP4 the cumulative regret on the order of $O(T^{3/4})$ was possible.\footnote{EXP4.R becomes the same as EXP4 when we set all risks and $\beta$ to zeros.}

The difficult of achieving a high probability statement with near-optimal cumulative regret $O(\sqrt{TK\ln(|\Pi|)})$ and cumulative constraint violation rate $\tilde{O}(T^{3/4})$ (the combination that matches to the state-of-art in the full information setting) is from the  Lagrangian dual variable $\lambda$. The variance of $\hat{z}_t[i]$ is proportional to $1/p_t[a_t]$. With $\lambda_t$, the variance of $\lambda_t \hat{z}_t[i]$ scales as $\lambda_t^2 /p_t[a_t]$.  As we show in Lemma~\ref{lemma:lambda_upper_bound} in Appendix~\ref{sec:analysis_exp4_P_R}, $\lambda_t$ could be as large as $|\beta|/(\delta\mu)$. Depending on the value of $\delta, \mu$, $\lambda_t$ could be large, e.g., $\Theta(\sqrt{T})$ if $\delta$ is a constant and $\mu = \Theta(1/\sqrt{T})$. Hence compared to EXP4.P, the Lagrangian dual variable in EXP4.P.R makes it more difficult to control the variance of $\tilde{x}_t$, which is an unbiased estimation of $y_t[i]+\lambda_t z_t[i]$. This is exactly where the trade-off $\epsilon$ comes from: we can tune the magnitude of $\delta$ to control the variance of $\tilde{x}_t$ and further control the trade-off between regret and risk violation . How to achieve total regret $O(\sqrt{TK\ln(|\Pi|)})$ and cumulative constraint violation $O(T^{3/4})$ in high probability is still an open problem.


\section{Conclusion}
In this work we study the problem of adversarial contextual bandits with adversarial risk constraints. We introduce the concept of risk constraints for arms and the goal is to satisfy the long-term risk constraint while achieve near-optimal regret in terms of reward.  The proposed two algorithm, EXP4.R and EXP4.P.R, are built on the existing EXP4 and EXP4.P algorithms. EXP4.R achieves near-optimal regret
and satisfies the long-term constraint in expectation while EXP4.P.R achieves similar theoretical bounds with high probability. We introduced a tradeoff in the analysis of EXP4.P.R which shows that one can trade the constraint violation for regret and vice versa. The regret bound and the constraint bound of EXP4.P.R does not match the state-of-art results of online learning with constraints due to the fact that the Lagrange dual parameter in worst case can significantly increase the variance of the algorithm. 

Same as EXP4 and EXP4.P, the computational complexity of a simple implementation of our algorithms per step is linear with respect to the size of the policy class. This drawback makes it difficulty to directly apply our algorithms to huge policy classes. Directly designing computational efficient algorithms for risk-aware adversarial contextual bandits might be hard, but one interesting future direction is to look into the hybrid case (i.e., i.i.d contexts but adversarial rewards and risks). In the hybrid case, it maybe possible to design computational efficient algorithms by leveraging the recent work in \citep{rakhlin2016bistro}
\medskip

\bibliographystyle{abbrvnat}
\bibliography{UAI2015}

\begin{thebibliography}{21}
\providecommand{\natexlab}[1]{#1}
\providecommand{\url}[1]{\texttt{#1}}
\expandafter\ifx\csname urlstyle\endcsname\relax
  \providecommand{\doi}[1]{doi: #1}\else
  \providecommand{\doi}{doi: \begingroup \urlstyle{rm}\Url}\fi

\bibitem[Agarwal et~al.(2014)Agarwal, Hsu, Kale, Langford, Li, and
  Schapire]{agarwal2014taming}
A.~Agarwal, D.~Hsu, S.~Kale, J.~Langford, L.~Li, and R.~Schapire.
\newblock Taming the monster: A fast and simple algorithm for contextual
  bandits.
\newblock In \emph{Proceedings of The 31st International Conference on Machine
  Learning}, pages 1638--1646, 2014.

\bibitem[Agrawal and Devanur(2015)]{agrawal2015linear}
S.~Agrawal and N.~R. Devanur.
\newblock Linear contextual bandits with knapsacks.
\newblock \emph{arXiv preprint arXiv:1507.06738}, 2015.

\bibitem[Agrawal et~al.(2015)Agrawal, Devanur, and Li]{agrawal2015contextual}
S.~Agrawal, N.~R. Devanur, and L.~Li.
\newblock Contextual bandits with global constraints and objective.
\newblock \emph{arXiv preprint arXiv:1506.03374}, 2015.

\bibitem[Auer et~al.(2002{\natexlab{a}})Auer, Cesa-Bianchi, and
  Fischer]{auer2002finite}
P.~Auer, N.~Cesa-Bianchi, and P.~Fischer.
\newblock Finite-time analysis of the multiarmed bandit problem.
\newblock \emph{Machine learning}, 47\penalty0 (2-3):\penalty0 235--256,
  2002{\natexlab{a}}.

\bibitem[Auer et~al.(2002{\natexlab{b}})Auer, Cesa-Bianchi, Freund, and
  Schapire]{auer2002nonstochastic}
P.~Auer, N.~Cesa-Bianchi, Y.~Freund, and R.~E. Schapire.
\newblock The nonstochastic multiarmed bandit problem.
\newblock \emph{SIAM Journal on Computing}, 32\penalty0 (1):\penalty0 48--77,
  2002{\natexlab{b}}.

\bibitem[Badanidiyuru et~al.(2013)Badanidiyuru, Kleinberg, and
  Slivkins]{badanidiyuru2013bandits}
A.~Badanidiyuru, R.~Kleinberg, and A.~Slivkins.
\newblock Bandits with knapsacks.
\newblock In \emph{Foundations of Computer Science (FOCS), 2013 IEEE 54th
  Annual Symposium on}, pages 207--216. IEEE, 2013.

\bibitem[Badanidiyuru et~al.(2014)Badanidiyuru, Langford, and
  Slivkins]{badanidiyuru2014resourceful}
A.~Badanidiyuru, J.~Langford, and A.~Slivkins.
\newblock Resourceful contextual bandits.
\newblock In \emph{COLT}, pages 1109--1134, 2014.

\bibitem[Beck and Teboulle(2003)]{beck2003mirror}
A.~Beck and M.~Teboulle.
\newblock Mirror descent and nonlinear projected subgradient methods for convex
  optimization.
\newblock \emph{Operations Research Letters}, 31\penalty0 (3):\penalty0
  167--175, 2003.

\bibitem[Beygelzimer et~al.(2011)Beygelzimer, Langford, Li, Reyzin, and
  Schapire]{beygelzimer2011contextual}
A.~Beygelzimer, J.~Langford, L.~Li, L.~Reyzin, and R.~E. Schapire.
\newblock Contextual bandit algorithms with supervised learning guarantees.
\newblock In \emph{AISTATS}, pages 19--26, 2011.

\bibitem[Bubeck et~al.(2012)Bubeck, Cesa-Bianchi, et~al.]{bubeck2012regret}
S.~Bubeck, N.~Cesa-Bianchi, et~al.
\newblock Regret analysis of stochastic and nonstochastic multi-armed bandit
  problems.
\newblock \emph{Foundations and Trends{\textregistered} in Machine Learning},
  5\penalty0 (1):\penalty0 1--122, 2012.

\bibitem[Bubeck et~al.(2015)]{bubeck2015convex}
S.~Bubeck et~al.
\newblock Convex optimization: Algorithms and complexity.
\newblock \emph{Foundations and Trends{\textregistered} in Machine Learning},
  8\penalty0 (3-4):\penalty0 231--357, 2015.

\bibitem[Ding et~al.(2013)Ding, Qin, Zhang, and Liu]{ding2013multi}
W.~Ding, T.~Qin, X.-D. Zhang, and T.-Y. Liu.
\newblock Multi-armed bandit with budget constraint and variable costs.
\newblock In \emph{AAAI}, 2013.

\bibitem[Jenatton et~al.(2016)Jenatton, Huang, and
  Archambeau]{jenatton2015adaptive}
R.~Jenatton, J.~Huang, and C.~Archambeau.
\newblock Adaptive algorithms for online convex optimization with long-term
  constraints.
\newblock \emph{ICML}, 2016.

\bibitem[Langford and Zhang(2008)]{langford2008epoch}
J.~Langford and T.~Zhang.
\newblock The epoch-greedy algorithm for multi-armed bandits with side
  information.
\newblock In \emph{Advances in neural information processing systems}, pages
  817--824, 2008.

\bibitem[Madani et~al.(2004)Madani, Lizotte, and Greiner]{madani2004budgeted}
O.~Madani, D.~J. Lizotte, and R.~Greiner.
\newblock The budgeted multi-armed bandit problem.
\newblock In \emph{International Conference on Computational Learning Theory},
  pages 643--645. Springer, 2004.

\bibitem[Mahdavi et~al.(2012)Mahdavi, Jin, and Yang]{mahdavi2012trading}
M.~Mahdavi, R.~Jin, and T.~Yang.
\newblock Trading regret for efficiency: online convex optimization with long
  term constraints.
\newblock \emph{The Journal of Machine Learning Research}, 13\penalty0
  (1):\penalty0 2503--2528, 2012.

\bibitem[Mannor et~al.(2009)Mannor, Tsitsiklis, and Yu]{mannor2009online}
S.~Mannor, J.~N. Tsitsiklis, and J.~Y. Yu.
\newblock Online learning with sample path constraints.
\newblock \emph{The Journal of Machine Learning Research}, 10:\penalty0
  569--590, 2009.

\bibitem[Rakhlin and Sridharan(2016)]{rakhlin2016bistro}
A.~Rakhlin and K.~Sridharan.
\newblock Bistro: An efficient relaxation-based method for contextual bandits.
\newblock \emph{ICML}, 2016.

\bibitem[Shalev-Shwartz(2011)]{Shwartz2011_FTML}
S.~Shalev-Shwartz.
\newblock {Online Learning and Online Convex Optimization}.
\newblock \emph{Foundations and Trends® in Machine Learning}, 4\penalty0
  (2):\penalty0 107--194, 2011.

\bibitem[Syrgkanis et~al.(2016)Syrgkanis, Luo, Krishnamurthy, and
  Schapire]{syrgkanis2016improved}
V.~Syrgkanis, H.~Luo, A.~Krishnamurthy, and R.~E. Schapire.
\newblock Improved regret bounds for oracle-based adversarial contextual
  bandits.
\newblock \emph{arXiv preprint arXiv:1606.00313}, 2016.

\bibitem[Zinkevich(2003)]{Zinkevich2003_ICML}
M.~Zinkevich.
\newblock {Online Convex Programming and Generalized Infinitesimal Gradient
  Ascent}.
\newblock In \emph{International Conference on Machine Learning (ICML 2003)},
  pages 421--422, 2003.

\end{thebibliography}

\newpage
\onecolumn
\appendix
\section*{Appendix}

\section{Proof of Proposition~\ref{prop:wrong_decision_set}}
\label{sec:prop_wrong_decision_set}
\begin{proof}
The proof is mainly about adapting the specific two-player game presented in \citep{mannor2009online} to the general online convex programming setting with adversarial constraints. We closely follow the notations in the example from Proposition 4 in  \citep{mannor2009online}.

Let us define the decision set $\mathcal{X} = \Delta([1,2])$, namely a 2-D simplex. We design two different loss functions: $\ell^1(x) = [-1,0]x$, and $\ell^2(x) = [-1,1]x$ (here $[a,b]$ stands for a 2-d row vector and hence $[a,b] x$ stands for the regular vector inner product). We also design two different constraints as: $f^1(x) = [-1,-1]x\leq 0$ and $f^2(x) = [1,-1]x\leq 0$. Note that both $\ell$ and $f$ are linear functions with respect $x$, hence they are convex loss functions and constraints with respect to $x$. The adversary picks loss functions among $\{\ell^1, \ell^2\}$ and constraints among $\{f^1, f^2\}$ and will generate the following sequence of loss functions and constraints. Initialize a counter $k = 0$, then:
\begin{enumerate}
	\item while $k=0$ or $\frac{1}{t-1}\sum_{i=1}^{t-1} x_i[1] > 3/4$, the adversary set $\ell_t = \ell^2(x)$ and $f_t = f^2(x)$, and set $k := k + 1$.
	\item For next $k$ steps, the adversary set $\ell_t = \ell^1(x)$ and $f_t = f^1(x)$. Then reset $k = 0$ and go back to step 1.
\end{enumerate}

For any time step $t$, let us define $\hat{q}_t = \frac{1}{t} \sum_{i=1}^t \mathbbm{1}(f_i = f^2)$, namely the fraction of the adversary picking the second type of constraint. Let us define $\hat{\alpha}_t = \sum_{i=1}^t x_i[1] / t$. Given any $\hat{q}_t$, we see that $\mathcal{O}'$ can be defined as
\begin{align}
\mathcal{O}' &= \{x\in \Delta([1,2]):  \hat{q}_t [1,-1]x + (1-\hat{q}_t)[-1,-1]x\leq 0)\}  \nonumber\\
& = \{x\in\Delta([1,2]): [2\hat{q}_t - 1, -1]x \leq 0\} = \{x\in \Delta([1,2]): 2\hat{q}_t x[1] - 1 \leq 0\},
\end{align} and the minimum loss the learner can get in hindsight with decisions restricted to $\mathcal{O}'$ is:
\begin{align}
r_t^{min}& = \min_{x\in\mathcal{O}'} (1-\hat{q}_t)[-1,0]x + \hat{q}_t [-1,1]x \nonumber\\
& = \begin{cases}  -1 & 0\leq \hat{q}_t \leq 1/2 \\
				-1/2 - 1/(2\hat{q}_t) + \hat{q}_t & 1/2\leq \hat{q}_t \leq 1 \end{cases}
\end{align}
The cumulative constraint violation at time step $t$ can be computed as $\sum_{i=1}^t f_i(x_i) =\sum_{i=1}^t \mathbbm{1}(f_i= f^1) [-1,-1]x_i + \mathbbm{1}(f_i = f^2)[1,-1]x_i$. We want to show that no matter what strategy the learner uses, as long as $\frac{1}{t}\limsup_{i\to\infty}\sum_i f_i(x_i) \leq 0$, we will have  $\limsup_{t\to\infty}(\sum_{i=1}^t \ell_i(x_i)/t)- r^{min}_t > 0$.

Following a similar argument from \citep{mannor2009online}, we can show that Step 2 is entered an infinite number of times. To show this, assume that step 2 only enters finite number of times. Hence as the game keeps staying in Step 1, the fraction of the adversary picking the second constraint $f^2$ approaches to one ($\hat{q}_t\to 1$), we will have as $t$ approaches to infinity,
\begin{align}
&\lim_{t\to\infty}\frac{1}{t}\sum_{i=1}^t f_i(x_i) = \lim_{t\to\infty}\frac{1}{t}\sum_{i=1}^t \mathbbm{1}(f_i= f^1) [-1,-1]x_i + \frac{1}{t}\sum_{i=1}^t\mathbbm{1}(f_i = f^2)[1,-1]x_i \nonumber\\
&= \lim_{t\to\infty}\frac{1}{t}\sum_{i=1}^t\mathbbm{1}(f_i = f^2)[1,-1]x_i = \lim_{t\to\infty} \frac{1}{t}\sum_{i=1}^t[1,-1]x_i = \lim_{t\to\infty}[1,-1](\frac{1}{t}\sum_{i=1}^t x_i ).
\end{align} Since $\sum_{i=1}^t x_i / t \in \Delta([1,2])$, we must have $\hat{\alpha}_t = \sum_{i=1}^t x_i[1]/t <= 1/2$ to ensure that the long-term constraint is satisfied: $\lim_{t\to\infty}\frac{1}{t}\sum_{i=1}^t f_i(x_i)\leq 0$. But when $\hat{\alpha}_t \leq 1/2$, the condition of entering Step 1 is violated and we must enter step 2. Hence step 2 is entered infinite number of times.  In particular, there exist infinite sequences $t_i$ and $t_i'$ such that $t_i < t_i' < t_{t+1}$, and the adversary picks $f^2, \ell^2$ in $(t_i, t_i']$ (Step 1) and the adversary picks $f^1, \ell^1$ in $(t_i', t_{i+1}]$ (Step 2). Since step 1 and step 2 executes the same number of steps (i.e., using the counter $k$'s value), we must have $\hat{q}_{t_i} = 1/2$ and $r^{min}_{t_i} = 1$. Furthermore, we must have $t_i' \geq t_{t+1}/2$. Note that $\hat{\alpha}_{t_i'} \leq 3/4$ since otherwise the adversary would be in step $1$ at time $t_i'+1$.  Thus, during the first $t_{i+1}$ steps, we must have:
\begin{align}
\sum_{j=1}^{t_{i+1}} x_j[1]  =  \sum_{j=1}^{t_i'} x_j[1] + \sum_{j=t_i'+1}^{t_{i+1}} x_j[1]  \leq \frac{3}{4} t_{i'} + (t_{i+1} - t_{i}') = t_{i+1} - t_i'/4 \leq \frac{7}{8} t_{i+1}.
\end{align}
It is easy to verify that $\frac{1}{t_{i+1}}\sum_{t=1}^{t_{i+1}} \ell_t(x_t) \geq -\frac{1}{t_{i+1}}\sum_{t=1}^{t_{i+1}} x_t[1] \geq -\frac{7}{8}$.
Hence, simply let $i\to\infty$, we have:
\begin{align}
\limsup_{t\to\infty} (\frac{1}{t}\sum_{i=1}^t \ell_i(x_i) - r^{min}_{t}) \geq  -7/8 + 1 = 1/8.
\end{align} Namely, we have shown that for cumulative regret, regardless what sequence of decisions $x_1,...,x_t$ the learner has played, as long as it needs to satisfy $\lim\sup_{t\to\infty}\frac{1}{t}\sum_{i=1}^t f_i(x_i) \leq 0$, we must have:
\begin{align}
\limsup_{t\to\infty} \big(\sum_{i=1}^t \ell_i(x_i) - \min_{x\in\mathcal{O}'} \sum_{i=1}^t \ell_i(x) \big) \geq t/8 = \Omega(t).
\end{align} Hence we cannot guarantee to achieve no-regret when competing agains the decisions in $\mathcal{O}'$ while satisfying the long-term constraint.

\end{proof}

\section{Analysis of Alg.~\ref{alg:ol_constraints} and Proof Of Theorem~\ref{them:ol}}
\begin{myproof}[Proof of Theorem~\ref{them:ol}]
Since the algorithm runs online mirror descent on the sequence of loss $\{\mathcal{L}_t(x,\lambda_t)\}_t$ with respect to $x$, using the existing results of online mirror descent (Theorem 4.2 and Eq. 4.10 from \cite{bubeck2015convex}), we know that for the sequence of $\{x_t\}_t$:
\begin{align}
\label{eq:mirror_descent_analysis}
&\sum_{t=1}^T(\mathcal{L}_t(x_t, \lambda_t) - \mathcal{L}_t(x, \lambda_t))\leq \frac{D_R(x, x_1)}{\mu} + \frac{\mu}{2\alpha}\sum_{t=1}^T\|\nabla_{x}\mathcal{L}(x_t,\lambda_t)\|_{*}^2.
\end{align}
Also, we know that the algorithm runs online gradient ascent on the sequence of loss $\{\mathcal{L}_t(x_t, \lambda)\}_t$ with respect to $\lambda$, using the existing analysis of online gradient descent \citep{Zinkevich2003_ICML}, we have for the sequence of $\lambda_t$:
\begin{align}
\label{eq:ogd_analysis}
&\sum_{t=1}^T\mathcal{L}_t(x_t,\lambda) - \sum_{t=1}^T\mathcal{L}_t(x_t,\lambda_t) \leq \frac{1}{\mu}\lambda^2 + \frac{\mu}{2}\sum_{t=1}^T \big(\frac{\partial \mathcal{L}_t(w_t,\lambda_t)}{\partial \lambda_t}\big)^2,
\end{align}

Note that for $(\nabla_{\lambda}\mathcal{L}_t(x_t,\lambda_t))^2 = (f_t(x_t) - \delta\mu\lambda_t)^2\leq 2f_t^2(x_t) + 2\delta^2\mu^2\lambda_t^2 \leq 2D^2 + \delta^2\mu^2\lambda_t^2$. Similarly for $\|\nabla_x\mathcal{L}_t(x_t,\lambda_t)\|_*^2$, we also have:
\begin{align}
&\|\nabla_x\mathcal{L}_t(x_t,\lambda_t)\|_*^2 \leq 2\|\nabla\ell_t(x_t)\|_*^2 + 2\|\lambda_t\nabla f_t(x_t)\|_*^2 \leq 2G^2(1+\lambda_t^2),
\end{align} where we first used triangle inequality for $\|\nabla_x\mathcal{L}_t(x_t,\lambda_t)\|_*$ and then use the inequality of $2ab\leq a^2+b^2, \forall a,b\in \mathcal{R}^+$. We also assume that the norm of the gradients are bounded as $max(\|\nabla\ell_t(x_t)\|_*, \|\nabla f_t(x_t)\|_*) \leq G\in\mathcal{R}^+$.
Now sum Inequality~\ref{eq:mirror_descent_analysis} and \ref{eq:ogd_analysis} from $t=0$ to $T$, we get:
\begin{align}
&\sum_{t}\mathcal{L}_t(x_t,\lambda) - \mathcal{L}_t(x,\lambda_t)\nonumber\\
&\leq\frac{2D_R(x,x_0) + \lambda^2}{2\mu} + \sum_{t}\mu(D^2 +  \delta^2\mu^2\lambda_t^2) + \sum_t\frac{\mu G^2}{\alpha}(1+\lambda_t^2) \nonumber\\
& = \frac{2D_R(x,x_0) + \lambda^2}{2\mu} + T\mu(D^2 + \frac{ G^2}{\alpha})  + \mu(\delta^2\mu^2 + \frac{ G^2}{\alpha})\sum\lambda_t^2.
\end{align}

Using the saddle-point convex and concave formation for $\mathcal{L}_t$, we have:
\begin{align}
&\sum_t\mathcal{L}_t(x_t, \lambda) - \mathcal{L}_t(x,\lambda_t) = \sum_t(\ell_t(x_t) - \ell_t(x)) + \sum_t(\lambda f_t(x_t) - \lambda_t f_t(x)) + \frac{\delta\mu}{2}\sum\lambda_t^2 - \frac{\delta\mu T}{2}\lambda^2 \nonumber \\
& \leq \frac{2B + \lambda^2}{2\mu} + T\mu(D^2 + \frac{ G^2}{\alpha}) + \mu(\delta^2\mu^2 + \frac{ G^2}{\alpha})\sum\lambda_t^2.
\end{align} Note that based on the setting of $\delta$ and $\mu$, we can show that $\delta\geq \delta^2\mu^2 + G^2/\alpha$. This is because $\delta^2\mu^2 + G^2/\alpha = \frac{4G^4 B}{\alpha^2 T(D^2+G^2/\alpha)} +G^2/\alpha \leq \frac{4G^2B}{T\alpha} + G^2/\alpha \leq 2G^2/\alpha$, where we assume that $T$ is large enough such that $T\geq 4B$.\footnote{Note that here for analysis simplicity we consider asymptotic property of the algorithm and assume $T$ is large enough and particularly larger than any constant. We don't necessarily have to assume $T\geq 4B$ here because we can explicitly solve the inequality $\delta\geq \delta^2\mu^2+G^2/\alpha$ to find the valid range of $\delta$, as \citep{mahdavi2012trading} did. }

Since we have $\delta \geq \delta^2\mu^2 + G^2/\alpha$, we can remove the term $\sum_{t}\lambda_t^2$ in the above inequality.
\begin{align}
&\sum_t(\ell_t(x_t) - \ell_t(x)) + \sum_t(\lambda f_t(x_t)  - \lambda_t f_t(x))- (\frac{\delta\mu T}{2} + \frac{1}{2\mu})\lambda^2 \leq \frac{2B}{2\mu} + T\mu(D^2 + G^2/\alpha).
\end{align}
Now set $x = x^*$, and set $\lambda = 0$, since $f_t(x^*)\leq 0$ for all $t$, we get:
\begin{align}
\sum_t(\ell_t(x_t) - \ell_t(x^*)) \leq \frac{2B}{2\mu} + T\mu(D^2 + G^2/\alpha) \leq 2\sqrt{BT(D^2+G^2/\alpha)},
\end{align} where we set $\mu = \sqrt{B/(T(D^2+G^2/\alpha))}$.

To upper bound $\sum_t f_t(x_t)$, we first note that we can lower bound $\sum_{t=1}^T (\ell_t(x_t) - \ell_t(x))$ as $\sum_{t=1}^T(\ell_t(x_t) - \ell_t(x))\geq -2FT$. Now let us assume that $\sum_t f_t(x_t)> 0$ (otherwise we are done). We set $\lambda = (\sum_t f_t(x_t))/(\delta\mu T + 1/\mu)$, we have:
\begin{align}
&\frac{(\sum_t f_t(x_t))^2}{2\delta\mu T + 1/\mu} \leq \frac{2B}{2\mu} + T\mu(D^2 + G^2/\alpha) + \sum_t(\ell_t(x^*) - \ell_t(x_t)) \nonumber \\
&\leq 2\sqrt{BT(D^2+G^2/\alpha)}  + 2FT
\end{align}

Substitute $\mu = \sqrt{B/(T(D^2+G^2/\alpha))}$ into the above inequality, we have:
\begin{align}
&(\sum_{t=1}^T f_t(x_t))^2 \leq 2\sqrt{BT(D^2+G^2/\alpha)}(2\delta\mu T + 1/\mu) + 2FT(2\delta\mu T + 1/\mu) \nonumber \\
& \leq \frac{8G^2}{\alpha}B T + 2T(D^2+\frac{D^2}{\alpha}) + 2T(D^2+\frac{G^2}{\alpha}) + T^{3/2}\sqrt{8F^2G^2/\alpha}.
\end{align} Take the square root on both sides of the above inequality and observe that $T^{3/2}\sqrt{8F^2G^2/\alpha}$ dominates the RHS of the above inequality, we prove the theorem.
\end{myproof}

\section{Analysis of EXP4.R}
\label{sec:exp4_R_analysis}
In this section we provide the full proof of theorem~\ref{them:bandits}.
\begin{myproof}[Proof of Theorem~\ref{them:bandits}]
We first present several known facts. First we have that for $w_t^T\hat{z}_t$:
\begin{align}
\label{eq:fact_1}
w_t^T\hat{z}_t = \mathbb{E}_{i\sim w_t}\hat{z}_t[i] = \mathbb{E}_{i\sim w_t}\pi_i(s_t)^T\hat{r}_t = \mathbb{E}_{i\sim w_t}\mathbb{E}_{j\sim \pi_i(s_t)}\hat{r}_t[j] = \mathbb{E}_{j\sim p_t}\hat{r}_t[j] = r_t[a_t] \leq 1.
\end{align} For $w_t^T\hat{y}_t$, we have:
\begin{align}
\label{eq:fact_2}
w_t^T\hat{y}_t = \mathbb{E}_{i\sim w_t}\hat{y}_t[i] = \mathbb{E}_{i\sim w_t}\pi_i(s_t)^T\hat{c}_t = \mathbb{E}_{j\sim p_t}\hat{c}_t[j] = c_t[a_t] \leq 1.
\end{align}

For $\mathbb{E}_{a_t\sim p_t}(w_t^T\hat{z}_t - \beta)^2$, we then have:
\begin{align}
\mathbb{E}_{a_t\sim p_t}(w_t^T\hat{z}_t - \beta)^2 = \mathbb{E}_{a_t\sim p_t}(r_t[a_t] - \beta)^2 \leq \mathbb{E}_{a_t}2 r_t[a_t]^2 + 2\beta^2\leq 4.
\end{align}
For $\mathbb{E}_{a_t\sim p_t} \hat{y}_t$, we have:
\begin{align}
\mathbb{E}_{a_t\sim p_t} \hat{y}_t[j] = \pi_j(s_t)^T\mathbb{E}_{a_t\sim p_t} \hat{c}_t = \pi_j(s_t)^Tc_t = y_t[j],
\end{align} which gives us $\mathbb{E}_{a_t\sim p_t}\hat{y}_t = y_t$. Similarly we can easily verify that $\mathbb{E}_{a_t\sim p_t}\hat{z}_t = z_t$.

For $\sum_{i=1}^{|\Pi|}w_t[i]\hat{y}_t[i]^2$, we have:
\begin{align}
&\sum_{i=1}^{|\Pi|}w_t[i]\hat{y}_t[i]^2 = \mathbb{E}_{i\sim w_t}\hat{y}_t[i]^2 = \mathbb{E}_{i\sim w_t}(\pi_j(s_t)^T\hat{c}_t)^2 = \mathbb{E}_{i\sim w_t}(\mathbb{E}_{j\sim \pi_i(s_t)}\hat{c}_t[j])^2 \nonumber \\
& \leq \mathbb{E}_{i\sim w_t}\mathbb{E}_{j\sim\pi_{i}(s_t)}(\hat{c}_t[j])^2 = \mathbb{E}_{j\sim p_t}(\hat{c}_t[j])^2 = \frac{c_t[a_t]^2}{p_t[a_t]}.
\end{align} Hence, for $\mathbb{E}_{a_t\sim p_t}\sum_{i=1}^{|\Pi|}w_t[i]\hat{y}_t[i]^2$ we have:
\begin{align}
\label{eq:fact_3}
\mathbb{E}_{a_t\sim p_t}\sum_{i=1}^{|\Pi|}w_t[i]\hat{y}_t[i]^2 \leq \mathbb{E}_{a_t\sim p_t} \frac{c_t[a_t]^2}{p_t[a_t]} = \sum_{k=0}^K c_t[k]^2 \leq K.
\end{align}
Similarly, for $\sum_{i=1}^{|\Pi|}w_t[i]\hat{z}_t[i]^2$, we have:
\begin{align}
\sum_{i=1}^{|\Pi|}w_t[i]\hat{z}_t[i]^2 = \mathbb{E}_{i\sim w_t}(\pi_i(s_t)^T\hat{r}_t)^2 \leq \mathbb{E}_{j\sim p_t}(\hat{r}_t[j])^2 = \frac{r_t[a_t]^2}{p_t[a_t]},
\end{align} and
\begin{align}
\label{eq:fact_4}
\mathbb{E}_{a_t\sim p_t}\sum_{i=1}^{|\Pi|}w_t[i]\hat{z}_t[i]^2 \leq K.
\end{align}
Now we are going to take expectation with respect to the randomized decisions $\{a_i\}$ on both sides of Inequality.~\ref{eq:before_expectation_tmp}. Fix time step $t$, conditioned on $a_1,...,a_{t-1}$, we have:
\begin{align}
&\mathbb{E}_{a_t} \Big[ \mathcal{L}_t(w_t,\lambda) -  \mathcal{L}_t(w,\lambda_t) \Big] \nonumber \\
&= \mathbb{E}_{a_t} \Big[ c_t[a_t] + \lambda(r_t[a_t] - \beta) - \frac{\delta\mu}{2}\lambda^2 - \hat{y}_t^T w - \lambda_t(\hat{z}_t^Tw - \beta) + \frac{\delta\mu}{2}\lambda_t^2  \Big] \nonumber \\
& = \mathbb{E}_{a_t}c_t[a_t] + \lambda (\mathbb{E}_{a_t}r_t[a_t] - \beta) - \frac{\delta\mu}{2}\lambda^2 - y_t^Tw - \lambda_t(z_t^Tw - \beta) + \frac{\delta\mu}{2}\lambda_t^2 .\nonumber \\
& \;\;\;\;\;\;\; (\textit{Used fact that $\mathbb{E}_{a_t\sim p_t}\hat{y}_t = y_t$ and $\mathbb{E}_{a_t\sim p_t}\hat{z}_t = z_t$ } ) \nonumber
\end{align}
Take the expectation with respect to $a_1,...,a_T$ on the LHS of Inequality~\ref{eq:before_expectation_tmp}, we have:
\begin{align}
\label{eq:fact_5}
&\mathbb{E}_{\{a_t\}_t}\sum_{t=1}^T\Big[ \mathcal{L}_t(w_t,\lambda) -  \mathcal{L}_t(w,\lambda_t) \Big] = \sum_{t=1}^T\mathbb{E}_{a_1,...,a_{t-1}}\mathbb{E}_{a_t|a_1,...,a_{t-1}}\Big[ \mathcal{L}_t(w_t,\lambda) -  \mathcal{L}_t(w,\lambda_t) \Big] \nonumber\\
& = \sum_{t=1}^T\big[\mathbb{E}c_t[a_t] + \lambda(\mathbb{E}r_t[a_t] - \beta)  - y_t^Tw - \lambda_t(z_t^Tw - \beta) + \frac{\delta\mu}{2}\lambda_t^2\big] -  \frac{\delta\mu T}{2}\lambda^2
\end{align}
Now take the expectation with respect to $a_1,...,a_T$ on the RHS of Inequality~\ref{eq:before_expectation_tmp} (we use $\mathbb{E}_{a_t|-a_t}$ to represent the expectation over the distribution of $a_t$ conditioned on $a_1,...,a_{t-1}$), we have:
\begin{align}
\label{eq:fact_6}
&\frac{\lambda^2}{\mu} + \frac{\ln(|\Pi|)}{\mu} + \mu\sum_{t=1}^T\Big(\mathbb{E}_{a_t|a_{-t}}(\sum_{i=1}^{|\Pi|}w_t[i]\hat{y}_t[i]^2 + \lambda_t^2w_t[i]\hat{z}_t[i]^2 )    +\mathbb{E}_{a_t|a_{-t}}(w_t^T\hat{z}_t - \beta)^2 + \delta^2\mu^2\lambda_t^2    \Big) \nonumber \\
& \leq \frac{\lambda^2}{\mu} + \frac{\ln(|\Pi|)}{\mu} + \mu\sum_{t=1}^T\Big( K  + \lambda_t^2 K     +  4 + \delta^2\mu^2\lambda_t^2    \Big) \nonumber \\
& \;\;\;\;\;\; \textit{(Used Eq.~\ref{eq:fact_3} and \ref{eq:fact_4} )} \nonumber \\
&  = \frac{\lambda^2}{\mu} + \frac{\ln(|\Pi|)}{\mu} + \mu T(K+4) + \mu(K+\delta^2\mu^2)\sum_{t=1}^T\lambda_t^2.
\end{align} Note that based on the setting of $\delta$ and $\mu$, we can show that $\delta\geq 2K+2\delta^2\mu^2$. This is because $2K+2\delta^2\mu^2 = 2K +  18K^2\ln(|\Pi|)/(T(K+4)) \leq 2K + 18K\ln(|\Pi|)/T\leq 3K$, where for simplicity we assume that $T$ is large enough ($T\geq 18\ln(|\Pi|)$).

Chain Eq.~\ref{eq:fact_5} and ~\ref{eq:fact_6} together and get rid of the terms that have $\lambda_t$ (due to the fact that $\delta\geq 2K+2\delta^2\mu^2$) and rearrange terms, we get:
\begin{align}
\label{eq:fact_7}
&\mathbb{E}\big(\sum_{t=1}^T c_t[a_t] -\sum_{t=1}^T y_t^Tw\big) + \sum_{t=1}^T\big(\lambda(\mathbb{E}r_t[a_t] - \beta) - \lambda_t(z_t^Tw - \beta)\big) - (\frac{\delta\mu T}{2} + \frac{1}{\mu})\lambda^2 \nonumber \\
& \leq \frac{ \ln(|\Pi|)}{\mu} + \mu T(K+4).
\end{align} The above inequality holds for any $w$. 
Substitute $w^*$  into Eq.~\ref{eq:fact_7}, we get:
\begin{align}
&\mathbb{E}\big(\sum_{t=1}^T c_t[a_t] -\sum_{t=1}^T y_t^Tw^*\big) + \sum_{t=1}^T\lambda(\mathbb{E}r_t[a_t] - \beta) - (\frac{\delta\mu T}{2} + \frac{1}{\mu})\lambda^2 \nonumber \\
& \leq \frac{ \ln(|\Pi|)}{\mu} + \mu T(K+4) \nonumber.
\end{align} Now let us set $\lambda = 0$, for regret, we get:
\begin{align}
\label{eq:fact_8}
&\mathbb{E}\big(\sum_{t=1}^T c_t[a_t] -\sum_{t=1}^T y_t^Tw^*\big) \leq \ln(|\Pi|)/\mu + \mu T(K+4)\nonumber\\
& \leq 2\sqrt{\ln(|\Pi|) T(K+4)} = O(\sqrt{TK\ln(|\Pi|)}),
\end{align} where  $\mu = \sqrt{\ln(|\Pi|)/T(K+4)}$.

For constraints $\sum (\mathbb{E}r_t[a_t] - \beta)$, let us assume that $\sum \mathbb{E}(r_t[a_t] - \beta) > 0$ (otherwise we are done), and substitute $\lambda =  (\sum \mathbb{E}r_t[a_t] - \beta) / (\delta\mu T + 2/\mu)$ into inequality~\ref{eq:fact_8} (note that $\lambda > 0$). Using the fact that $\mathbb{E}\big(\sum_{t=1}^T c_t[a_t] -\sum_{t=1}^T y_t^Tw^*\big) \geq -2T$,  we get:
\begin{align}
(\sum_{t=1}^T(\mathbb{E}r_t[a_t] - \beta))^2 \leq (2\delta\mu T + 4/\mu)\big( 2T + 2\sqrt{\ln(|\Pi|) T (K+2+2\beta^2)}  \big)
\end{align} Substitute $\mu = \sqrt{\ln(|\Pi|)/T(K+4)}$ and $\delta = 3K$ back to the above equation, it is easy to verity that:
\begin{align}
\label{eq:fact_9}
(\sum_{t=1}^T(\mathbb{E} r_t[a_t] - \beta) )^2 \leq 12K\sqrt{\frac{\ln(|\Pi|)}{K+4}}T^{3/2} + 12K\ln(|\Pi|)T + 8T^{3/2}\sqrt{\frac{K+4}{\ln(|\Pi|)}} + 8T(K+4).
\end{align} Since we consider the asymptotic property when $T\to\infty$, we can see that the LHS of the above inequality is dominated by $\sqrt{K\ln(|\Pi|)}T^{3/2}$. Hence,
\begin{align}
(\sum_{t=1}^T(\mathbb{E} r_t[a_t] - \beta) )^2 \leq O(\sqrt{K\ln(|\Pi|)}T^{3/2}).
\end{align} Take the square root on both sides of the above inequality, we prove the theorem.
\end{myproof}

\section{Algorithm and Analysis of EXP4.P.R}
\label{sec:exp4_p_r}
\subsection{Algorithm}
We present the EXP4.P.R algorithm in Alg.~\ref{alg:exp4p_constraints}.
\begin{algorithm}
\caption{Exp4.P with Risk Constraints (EXP4.P.R)}
 \label{alg:exp4p_constraints}
\begin{algorithmic}[1]
  \STATE {\bfseries Input:} Policy Set $\Pi$
  \STATE Initialize $w_0 = [1/N,...,1/N]^T$ and $\lambda = 0$.
 \FOR {t = 0 to T}
    \STATE Receive context $s_t$.
    \STATE Query each experts to get the sequence of advice $\{\pi_i(s_t)\}_{i=1}^N$.
    \STATE Set $p_t = \sum_{i=1}^N w_t[i]\pi_i(s_t)$.
    \STATE Draw action $a_t$ randomly according to probability $p_t$.
    \STATE Receive cost $c_t[a_t]$ and risk $r_t[a_t]$.
    \STATE Set the cost vector $\hat{c}_t\in R^K$ and the risk vector $\hat{r}_t\in R^K$ as: \label{line:unbiased_cost_risk}
    \begin{align}
    \hat{c}_t[i] = \frac{c_t[i]\mathbbm{1}(a_t = i) }{p_t[i]}, \;\;\; \hat{r}_t[i] = \frac{r_t[i]\mathbbm{1}(a_t = i) }{p_t[i]},\forall i\in\{1,2,...,K\}.
    \end{align}
    \STATE For each expert $j$, set:  \label{line:unbiased_expert_cost_risk}
    \begin{align}
    \hat{y}_t[j] = \pi_j(s_t)^T\hat{c}_t,\;\;\; \hat{z}_t[j] = \pi_j(s_t)^T\hat{r}_t, \forall j\in\{1,2...,N\}.
    \end{align}
    \STATE Set $\tilde{x}_t = \hat{y}_t + \lambda_t \hat{z}_t$.
    \STATE Update $w_{t+1}$ as: \label{line:update_weight_p}
    \begin{align}
    w_{t+1}[i] = \frac{w_t[i]\exp(-\mu (\tilde{x}_t[i]- \kappa \sum_{k=1}^K \frac{\pi_i(s_t)[k]}{p_t[k]}))}{\sum_{j=1}^{|\Pi|}w_t[j]\exp(-\mu ( \tilde{x}_t[j]- \kappa \sum_{k=1}^K \frac{\pi_j(s_t)[k]}{p_t[k]}))}, \nonumber
    \end{align}
    \STATE Update $\lambda_{t+1}$ as: \label{line:update_dual_p}
    \begin{align}
    \lambda_{t+1} = \max\{0, \lambda_t + \mu (w_t^T\hat{z}_t - \beta - \delta\mu\lambda_t)\}. \nonumber
     \end{align}
 \ENDFOR
\end{algorithmic}
\end{algorithm}

\subsection{Analysis of EXP4.P.R}
\label{sec:analysis_exp4_P_R}
We give detailed regret analysis of EXP4.P.R in this section. Let us define $\hat{x}_t(\lambda)$ as $\hat{x}_t(\lambda)[i] = \hat{y}_t[i] + \lambda\hat{z}_t[i] - \kappa\sum_{k=1}^K \frac{\pi_i(s_t)[k]}{p_t[k]}, \forall i\in [N]$ and $\mathcal{L}_t(w, \lambda) = w^T \hat{x}_t - \lambda\beta - \frac{\delta\mu}{2}\lambda^2$. As we can see that Line~\ref{line:update_weight_p} is essentially running Weighted Majority algorithm on the sequence of functions $\{\mathcal{L}_t(w, \lambda_t)\}_t$ while Line~\ref{line:update_dual_p} is running Online Gradient Ascent on the sequence of functions $\{\mathcal{L}_t(w_t, \lambda)\}_t$. Applying the classic analysis of Weighted Majority and analysis of Online Gradient Descent, we can show that:
\begin{lemma}
The sequences $\{w_t\}_t$ and $\{\lambda_t\}_t$ generated from Lines~\ref{line:update_weight_p} and~\ref{line:update_dual_p} in EXP4.P.R has the following property:
\begin{align}
&\sum_{t=1}^T\mathcal{L}_t(w_t,\lambda) - \sum_{t=1}^T\mathcal{L}_t(w,\lambda_t) \nonumber\\
& \leq \frac{\lambda^2}{\mu} + \frac{\ln(|\Pi|)}{\mu} + \frac{\mu}{2}\sum_{t=1}^T\Big(
\sum_{i=1}^{|\Pi|}w_t[i](\hat{x}_t(\lambda_t)[i])^2 + 2(w_t^T\hat{z}_t - \beta)^2 + 2\delta^2\mu^2\lambda_t^2
\Big).
\label{eq:before_expectation}
\end{align}
\end{lemma}
\begin{proof}
Using the classic analysis of Weighted Majority algorithm, we can get that for the sequence of loss $\{\mathcal{L}_t(w,\lambda_t)\}_t$:
\begin{align}
\sum_{t=1}^T\mathcal{L}_t(w_t,\lambda_t) - \sum_{t=1}^T\mathcal{L}_t(w,\lambda_t) \leq \frac{\ln(|\Pi|)}{\mu} + \frac{1}{2}\mu\sum_{t=1}^T\sum_{i=1}^{|\Pi|}w_t[i]\big(\hat{x}_t(\lambda_t)[i]\big)^2, \nonumber
\end{align} for any $w\in\mathcal{B}$. On the other hand, we know that we compute $\lambda_t$ by running Online Gradient Descent on the loss functions $\{\mathcal{L}_t(w_t,\lambda)\}_t$. Applying the classic analysis of Online Gradient Descent, we can get:
\begin{align}
\sum_{t=1}^T\mathcal{L}_t(w_t,\lambda) - \sum_{t=1}^T\mathcal{L}_t(w_t,\lambda_t) \leq \frac{1}{\mu}\lambda^2 + \frac{\mu}{2}\sum_{t=1}^T \big(\frac{\partial \mathcal{L}_t(w_t,\lambda_t)}{\partial \lambda_t}\big)^2, \nonumber
\end{align} for any $\lambda \geq 0$.

We know that $\partial \mathcal{L}_t(w_t,\lambda)/\partial \lambda_t = w_t^T\hat{z}_t - \beta - \delta\mu\lambda_t$. Substitute these gradient and derivatives back to the above two inequalities, and then sum the above two inequality together we get:
\begin{align}
&\sum_{t=1}^T\mathcal{L}_t(w_t,\lambda) - \sum_{t=1}^T\mathcal{L}_t(w,\lambda_t) \nonumber\\
&\leq  \frac{\lambda^2}{\mu} + \frac{\ln(|\Pi|)}{\mu} + \frac{\mu}{2}\sum_{t=1}^T\Big(
\sum_{i=1}^{|\Pi|}w_t[i](\hat{x}_t(\lambda_t)[i])^2 + (w_t^T\hat{z}_t - \beta - \delta\mu\lambda_t)^2
\Big) \nonumber\\
& \leq \frac{\lambda^2}{\mu} + \frac{\ln(|\Pi|)}{\mu} + \frac{\mu}{2}\sum_{t=1}^T\Big(
\sum_{i=1}^{|\Pi|}w_t[i](x_t(\lambda_t)[i])^2 + 2(w_t^T\hat{z}_t - \beta)^2 + 2\delta^2\mu^2\lambda_t^2
\Big),\nonumber
\end{align} where in the last ineqaulity we use the fact that $(a+b)^2\leq 2a^2+2b^2$, for any $a,b\in\mathbb{R}$.
\end{proof}

We first show that the Lagrangian dual parameter $\lambda_t$ can be upper bounded:
\begin{lemma}
\label{lemma:lambda_upper_bound}
Assume that $\delta \leq 1/\mu^2$. For any $t\in[T]$, we have $\lambda_t \leq \frac{|\beta|}{\delta\mu}$.
\end{lemma}
\begin{proof}
Remember that the update rule for $\lambda_t$ is defined as:
\begin{align}
\lambda_{t+1} = \max\{0, \lambda_t + \mu (w_t^T\hat{z}_t - \beta - \delta\mu\lambda_t)\}.
\end{align}  We prove the lemma by induction. For $t= 0$, since we set $\lambda_0 = 0$, we have $\lambda_0\leq (|\beta|/(\delta\mu)$. Now let us consider time step $t$  and assume that that $\lambda_t \leq (|\beta|)/(\delta\mu)$ for $\tau \leq t$. Note that $w_t^T \hat{z}_t = r_t[a_t] \leq 0$ and from the update rule of $\lambda$, we have:
\begin{align}
\lambda_{t+1} \leq \max\{0, \lambda_t + \mu(|\beta| - \delta\mu\lambda_t) \}
\end{align}
For the case when $\lambda_t = 0$, we have $\lambda_{t+1} = \mu |\beta|$.  Since we assume that $\delta\leq 1/\mu^2$, we can easily verify that $\lambda_{t+1} \leq \mu|\beta| \leq |\beta|/(\delta\mu)$.

For the case when $\lambda_t \geq 0$, since we see that $\lambda_t + \mu(|\beta| - \delta\mu\lambda_t) \geq 0$ from the induction hypothesis that $\lambda_t \leq |\beta|/(\delta\mu)$, we must have:
\begin{align}
\lambda_{t+1} = \lambda_t + \mu (|\beta| - \delta\mu\lambda_t).
\end{align} Subtract $|\delta|/\mu\beta$ on both sides of the above inequality, we get:
\begin{align}
\lambda_{t+1} - \frac{|\beta|}{\delta\mu} = (1-\delta\mu^2) \big(\lambda_t  - \frac{|\beta|}{\delta\mu}  \big)
\end{align} Since we have $\lambda_t\leq |\beta|/(\delta\mu)$ and $\delta \leq 1/\mu^2$, it is easy to see that we have for $\lambda_{t+1}$:
\begin{align}
\lambda_{t+1} - \frac{|\beta|}{\delta\mu}\leq 0.
\end{align} Hence we prove the lemma.
\end{proof}
For notation simplicity, let us denote $\frac{|\beta|}{\delta\mu}$ as $\lambda_m$. 

We now show how to relate $\sum_t \hat{y}[i] + \lambda_t\hat{z}[i] - \kappa\sum_{j=1}^K\frac{\pi_i(s_t)[j]}{p_t[j]}$ to $\sum_t y_t[i] + \lambda_t {z}[i]$ for any $i\in [|\Pi|]$:
\begin{lemma}
\label{lemma:high_prob}
In EXP4.P.R (Alg.~\ref{alg:exp4p_constraints}),  with probability at least $1-\delta$, for any $w\in\Delta{\Pi}$, we have:
\begin{align}
\sum_{t=1}^T\sum_{i=1}^{|\Pi|} &w[i](\hat{y}_t[i] + \lambda_t\hat{z}_t[i] - \kappa\sum_{j=1}^K\frac{\pi_i(s_t)[j]}{p_tj]})\nonumber \\
& \leq \sum_{t=1}^T\sum_{i=1}^{|\Pi|} (w[i](y_t[i] + \lambda_t z_t[i]) + (1+\lambda_m) \frac{\ln(|\Pi|/\delta)}{\kappa}. \nonumber
\end{align}
\end{lemma}
We use similar proof strategy as shown in the proof of Lemma 3.1 in \citep{bubeck2012regret} with three additional steps:(1) union bound over all polices in $\Pi$, (2) introduction of a distribution $w\in\Delta(\Pi)$, (3) taking care of $\lambda_t$ by using its upper bound from Lemma~\ref{lemma:lambda_upper_bound}.
\begin{proof}
Let us set $\delta' = \delta / |\Pi|$ and fix $i\in[|\Pi|]$. Define  $\tilde{x}_t(\lambda_t) =  \hat{y}_t + \lambda_t\hat{z}_t$ and we denote $\hat{x}_t(\lambda_t)[i] = \tilde{x}_t(\lambda_t)[i] - \kappa\sum_{j=1}^K (\pi_i(s_t)[j]/p_t[j])$.

For notation simplicity, we are going to use $\tilde{x}_t$ and $\hat{x}_t$ to represent $\tilde{x}_t(\lambda_t)[i]/(1+\lambda_m)$ and $\hat{x}_t(\lambda_t)[i]/(1+\lambda_m)$ respectively in the rest of the proof.

Let us also define $x_t = (y_t[i] + \lambda_t z_t[i])/(1+ \lambda_m)$.  It is also straightforward to check that $\kappa(\hat{x}_t - x_t) \leq 1$ since $\hat{x}_t \leq 0$, $-x_t\leq 1$ and $0<\kappa \leq 1$.
Note that it is straightforward to show that $\mathbb{E}_t(\tilde{x}_t) = x_t$, where we denote $\mathbb{E}_t$ as the expectation conditioned on randomness from $a_1,..., a_{t-1}$.

Following the same strategy in the proof of Lemma 3.1 in \citep{bubeck2012regret}, we can show that:
\begin{align}
&\mathbb{E}_t \big[ \exp(\kappa(\hat{x}_t - x_t) ) \big]  = \mathbb{E}_t \big[ \exp(\kappa(\tilde{x}_t - \kappa\sum_{j=1}^K(\pi_i(s_t)[j]/p_t[j]) - x_t  ) \big] \nonumber \\
&\leq ( 1 + \mathbb{E}_t\kappa(\tilde{x}_t - x_t) +   \kappa^2\mathbb{E}_t (\tilde{x}_t - x_t)^2)\exp(-\kappa^2\sum_{j=1}^K\frac{\pi_i(s_t)[j]}{p_t[j]})\nonumber \\
& \leq (1+ \kappa^2\mathbb{E}_t (\tilde{x}_t^2))\exp(-\kappa^2\sum_{j=1}^K\frac{\pi_i(s_t)[j]}{p_t[j]})
\label{eq:fact30}
\end{align}
We can upper bound $\mathbb{E}_t(\tilde{x}_t^2)$ as follows:
\begin{align}
&\mathbb{E}_t(\tilde{x}_t^2) = \mathbb{E}_t \Big[ \Big(\big(\sum_{j=1}^K\pi_i(s_t)[j]\frac{{c}_t[j]\mathbbm{1}(a_t = j)}{p_t[j]} + \lambda_t\sum_{j=1}^K\pi_i(s_t)[j]\frac{{r}_t[j]\mathbbm{1}(a_t = j)}{p_t[j]}\big) / (1+\lambda_m)\Big)^2  \Big]\nonumber \\
& \leq \mathbb{E}_{t, j\sim \pi_i(s_t)} \Big( \big(\hat{c}[j] /p_t[j] + \lambda_t\hat{r}_t[j]/p_t[j]\big) / (1+\lambda_m)   \Big)^2 \nonumber \\
& = \mathbb{E}_{j\sim \pi_i(s_t)} ((c_t[j] + \lambda_t r_t[j])/(1+\lambda_m))^2 / p_t[j]  \leq \mathbb{E}_{j\sim\pi_t(s_t)}(1/p_t(j)) = \sum_{j=1}^K\frac{\pi_i(s_t)[j]}{p_t[j]}
\end{align} where the first inequality comes from Jensen's inequality and the last inequality comes from the fact that $|c_t[j]|\leq1$ and $|\lambda_t r_t[j]|\leq \lambda_m$. Substitute the above results in Eq.~\ref{eq:fact30}, we get:
\begin{align}
&\mathbb{E}_t \big[ \exp(\kappa(\hat{x}_t - x_t) ) \big] \leq (1 + \kappa^2\sum_{j=1}^K\frac{\pi_i(s_t)[j]}{p_t[j]})\exp(-\kappa^2\sum_{j=1}^K\frac{\pi_i(s_t)[j]}{p_t[j]}) \nonumber \\
&\leq \exp(\kappa^2\sum_{j=1}^K\frac{\pi_i(s_t)[j]}{p_t[j]})\exp(-\kappa^2\sum_{j=1}^K\frac{\pi_i(s_t)[j]}{p_t[j]}) \leq 1.
\end{align}
Hence, we have:
\begin{align}
\mathbb{E}\exp(\kappa\sum_{t=1}(\hat{x}_t - x_t)) \leq 1.
\end{align} Now from Markov inequality we know $P(X \geq \ln(\delta^{-1})) \leq \delta\mathbb{E}(e^X)$. Hence, this gives us that with probability least $1-\delta$:
\begin{align}
\kappa \sum_t (\hat{x}_t - x_t)  \leq \ln(1/\delta).
\end{align} Substitute the representation of $\hat{x}_t, x_t$ in, we get for $i$, with probability $1-\delta'$:
\begin{align}
\sum_{t=1}^T \hat{y}_t[i] + \lambda_t \hat{z}_t[i] - \kappa\sum_{j=1}^K (\pi_i(s_t)[j]/p_t[j])\leq \sum_{t=1}^T y_t[i] + \lambda_t z_t[i] + (1+\lambda_m) \frac{\ln(1/\delta')}{\kappa}. \nonumber
\end{align} Now apply union  bound over all policies in $\Pi$, it is straightforward to show that for any $i\in |\Pi|$, with probability at least $1-\delta$, we have:
\begin{align}
\sum_{t=1}^T \hat{y}_t[i] + \lambda_t \hat{z}_t[i] - \kappa\sum_{j=1}^K (\pi_i(s_t)[j]/p_t[j]) \leq \sum_{t=1}^T y_t[i] + \lambda_t z_t[i] + (1+\lambda_m) \frac{\ln(|\Pi|/\delta)}{\kappa}. \nonumber
\end{align} To prove the lemma, now let us fix any $w\in\Delta(|\Pi|)$, we can simply multiple $w[i]$ on the both sides of the above inequality, and then sum over from $i=1$ to $|\Pi|$.
\end{proof}

Let us define $\hat{w} \in \Delta(\Pi)$ as:
\begin{align}
\hat{w} = \arg\min_{w\in\Delta{(\Pi)}}\sum_{t=1}^T \sum_{i=1}^{|\Pi|}w[i](\hat{y}[i] + \lambda_t\hat{z}[i] - \kappa\sum_{j=1}^K\frac{\pi_i(s_t)[j]}{p_t[j]}),
\end{align}  and $\hat{w}^*\in\Delta(\Pi)$ as:
\begin{align}
\hat{w}^* = \arg\min_{w\in\Delta(\Pi)} \sum_{t=1}^T \sum_{i=1}^{|\Pi|}w[i]({y}[i] + \lambda_t {z}[i])
\end{align}

Now we turn to prove Theorem~\ref{them:exp4_P_R}.

\begin{myproof}[Proof of Theorem~\ref{them:exp4_P_R}]
We prove the asymptotic property of Alg.~\ref{alg:exp4p_constraints} when $T$ approaches to infinity. Since we set $\mu = \sqrt{\frac{\ln(|\Pi|)}{(3K+4)T}}$ and $\delta = T^{-\epsilon+1/2}K$, we can first verify the condition $\delta\leq 1/\mu^2$ in Lemma~\ref{lemma:lambda_upper_bound}. This condition holds since $\delta = O(T^{0.5})$ while $1/\mu^2 = \Theta(T)$.

Let us first compute some facts.  For $w_t^T \hat{x}_t$, we have:
\begin{align}
&w_t^T\hat{x}_t(\lambda_t) = \mathbb{E}_{j\sim w_t} (\hat{y}_t[j] + \lambda_t\hat{z}_t[j] - \kappa\sum_{i=1}^K\frac{\pi_j(s_t)[i]}{p_t[i]})  = \mathbb{E}_{j\sim p_t}\hat{c}_t[j] + \lambda_t\mathbb{E}_{j\sim p_t}\hat{r}_t[j] - \kappa \mathbb{E}_{j\sim p_t}\frac{1}{p_t[j]} \nonumber \\
&= c_t[a_t] + \lambda_t r_t[a_t] - \kappa K.
\end{align}

For $\sum_{i=1}^{|\Pi|}w_t[i](\hat{x}_t(\lambda_t)[i])^2$, we have:
\begin{align}
\label{eq:fact_12}
&\sum_{i=1}^{|\Pi|}w_t[i](\hat{x}_t(\lambda_t)[i])^2  = \mathbb{E}_{i\sim w_t} (\hat{x}_t(\lambda_t)[i])^2 = \mathbb{E}_{i\sim w_t} (\hat{y}_t[i] + \lambda_t\hat{z}_t(i) - k\sum_{j=1}^K\frac{\pi_i(s_t)[j]}{p_t[j]})^2 \nonumber \\
& \leq \mathbb{E}_{i\sim w_t, j\sim \pi_i(s_t)} (\hat{c}_t[j] + \lambda_t\hat{r}_t[j] - \kappa/p_t[j])^2 = \mathbb{E}_{j\sim p_t} \big(\hat{c}_t[j] + \lambda_t\hat{r}_t[j] - \kappa/p_t[j]\big)^2 \nonumber \\
& = \sum_{i=1}^K p_t[i] \frac{(c_t[i]\mathbbm{1}(a_t = i) + \lambda_t r_t[i]\mathbbm{1}(a_t = i) - \kappa)^2}{p_t[i]^2} \nonumber \\
&= \sum_{i=1}^K\frac{(c_t[i]\mathbbm{1}(a_t = i) + \lambda_t r_t[i]\mathbbm{1}(a_t = i) - \kappa)^2}{p_t[i]}\nonumber \\
& \leq \sum_{i=1}^K (-1 - \lambda_t - \kappa)(\hat{c}_t[i] + \lambda_t\hat{r}_t[i] - \kappa/p_t[i]) \nonumber \\
& = K(-1-\lambda_t - \kappa)\sum_{i=1}^K ((1/K)\hat{c}_t[i] + \lambda_t(1/K)\hat{r}_t[i] - \kappa\frac{1/K}{p_t[i]}) \nonumber \\
&\leq K(-1-\lambda_t - \kappa)\big(\sum_{i=1}^{|\Pi|}\hat{w}[i](\hat{y}_t[i] + \lambda_t\hat{z}_t[i] - \kappa\sum_{j=1}^K\frac{\pi_{{i}}(s_t)[j]}{p_t[j]})    \big),
\end{align} where the first inequality comes from Jesen's inequality and the last inequality uses the assumption that the $\Pi$ contains the uniform policy (i.e., the policy that assign probability $1/K$ to each action).
Consider the RHS of Eq.~\ref{eq:before_expectation}, we have:
\begin{align}
& \frac{\lambda^2}{\mu} + \frac{\ln(|\Pi|)}{\mu} + \frac{\mu}{2}\sum_{t=1}^T\sum_{i=1}^{|\Pi|}w_t[i](\hat{x}_t(\lambda_t)[i])^2 + \frac{\mu}{2}\sum_{t=1}^T(w_t^T\hat{z}_t - \beta - \delta\mu\lambda_t)^2 \nonumber \\
&\leq  \frac{\lambda^2}{\mu} + \frac{\ln(|\Pi|)}{\mu}  + \frac{\mu}{2}\sum_{t=1}^T  K(-1-\lambda_t - \kappa) \Big(\sum_{i=1}^{|\Pi|}\hat{w}[i]\big(\hat{y}_t[i] + \lambda_t\hat{z}_t[i] - \kappa\sum_{j=1}^K\frac{\pi_{{i}}(s_t)[j]}{p_t[j]}\big)\Big)  +\mu\sum_{t=1}^T((w_t^T\hat{z}_t - \beta)^2  + \delta^2\mu^2\lambda_t^2) \nonumber \\
& =   \frac{\lambda^2}{\mu} + \frac{\ln(|\Pi|)}{\mu}  + \frac{\mu}{2}\sum_{t=1}^T  K(-1-\lambda_t - \kappa) \Big(\sum_{i=1}^{|\Pi|}\hat{w}[i]\big(\hat{y}_t[i] + \lambda_t\hat{z}_t[i] - \kappa\sum_{j=1}^K\frac{\pi_{{i}}(s_t)[j]}{p_t[j]}\big)\Big)  +\mu\sum_{t=1}^T((r_t[a_t] - \beta)^2  + \delta^2\mu^2\lambda_t^2)
\label{eq:fact20}
\end{align}
Consider the LHS of Eq.~\ref{eq:before_expectation}, set $w = \hat{w}$, we have:
\begin{align}
\label{eq:fact21}
&\sum_{t=1}^T\Big[\mathcal{L}_t(w_t,\lambda) - \mathcal{L}_t(\hat{w},\lambda_t)\Big] \nonumber \\
& = \sum_{t=1}^T \Big[ c_t[a_t] + \lambda r_t[a_t] - \kappa K  - \lambda\beta - \delta\mu\lambda^2/2 - \Big(\sum_{i=1}^{|\Pi|}\hat{w}[i]\big(\hat{y}_t[i] + \lambda_t\hat{z}_t[i] - \kappa\sum_{j=1}^K\frac{\pi_{{i}}(s_t)[j]}{p_t[j]}\big)\Big)   + \lambda_t\beta + \delta\mu\lambda_t^2/2\Big].
\end{align} Chaining Eq.~\ref{eq:fact20} and Eq.~\ref{eq:fact21} together and rearrange terms, we will get:
\begin{align}
\label{eq:fact22}
&\sum_{t=1}^T \Big[c_t[a_t]  + \lambda(r_t[a_t] - \beta)   + \lambda_t\beta + \delta\mu\lambda_t^2/2 \Big] - T\delta\mu\lambda^2/2 \nonumber \\
&\leq  T\kappa K + \frac{\lambda^2+\ln(|\Pi|)}{\mu}  + \sum_{t=1}^T (1 - \frac{\mu K}{2}(1+\lambda_t +\kappa))\Big(\sum_{i=1}^{|\Pi|}\hat{w}[i]\big(\hat{y}_t[i] + \lambda_t\hat{z}_t[i] - \kappa\sum_{j=1}^K\frac{\pi_{{i}}(s_t)[j]}{p_t[j]}\big)\Big)  \nonumber\\
&\;\;\;\;\;\;\; + \mu\sum_{t=1}^T (2 + 2\beta^2 + \delta^2\mu^2\lambda_t^2).
\end{align}
Since we have $\delta \geq \frac{|\beta|}{2/K - \mu - \kappa\mu}$, we can show that $1-\frac{\mu K}{2}(1+\lambda_t+\kappa) \geq 0$.

Now back to Eq.~\ref{eq:fact22}, using Lemma.~\ref{lemma:high_prob}, we have with probability $1-\nu$:
\begin{align}
&\sum_{t=1}^T \Big[c_t[a_t]  + \lambda(r_t[a_t] - \beta)   + \lambda_t\beta + \delta\mu\lambda_t^2/2 \Big] -T\delta\mu\lambda^2/2 \nonumber \\
&\leq  T\kappa K + \frac{\lambda^2+\ln(|\Pi|)}{\mu}  + \sum_{t=1}^T (1 - \frac{\mu K}{2}(1+\lambda_t +\kappa))\Big(\sum_{i=1}^{|\Pi|}\hat{w}^*[i](y_t[i] + \lambda_t z_t[i])\Big) \nonumber \\
&\;\;\;\;\;\;\; + (1+\lambda_m)\frac{\ln(|\Pi|/\nu)}{\kappa}  + (2+2\beta^2)T\mu + \mu^3\delta^2\sum_t \lambda_t^2 \nonumber \\
& \leq  T\kappa K + \frac{\lambda^2+\ln(|\Pi|)}{\mu}  + \sum_{t=1}^T (1 - \frac{\mu K}{2}(1+\lambda_t +\kappa))\Big(\sum_{i=1}^{|\Pi|}{w}^*[i](y_t[i] + \lambda_t z_t[i])\Big) \nonumber \\
&\;\;\;\;\;\;\; + (1+\lambda_m)\frac{\ln(|\Pi|/\nu)}{\kappa}  + (2+2\beta^2)T\mu + \mu^3\delta^2\sum_t \lambda_t^2.
\end{align} where the last inequality follows from the definition of $\hat{w}^*$ and ${w}^*$.  Rearrange terms, we get:
\begin{align}
&\sum_{t=1}^T \Big[ (c_t[a_t]  - {w^*}^Ty_t )+ \lambda(r_t[a_t] - \beta) - \lambda_t({w^*}^Tz_t - \beta)\big] - T\delta\mu\lambda^2/2 + \sum_{t=1}^T \delta\mu\lambda_t^2/2 \nonumber \\
& \leq T\kappa K + \frac{\lambda^2+\ln(|\Pi|)}{\mu} + \sum_{t=1}^T \frac{\mu K}{2}(1+\lambda_t + \kappa)(1+\lambda_t) + (1+\lambda_m)\frac{\ln(|\Pi|/\nu)}{\kappa}  + (2+2\beta^2)T\mu + \mu^3\delta^2\sum_t \lambda_t^2 \nonumber \\
& \leq T\kappa K + \frac{\lambda^2+\ln(|\Pi|)}{\mu} + \sum_{t=1}^T \frac{\mu K}{2} (1+(2+\kappa)\lambda_t + \kappa) + (1+\lambda_m)\frac{\ln(|\Pi|/\nu)}{\kappa} + (2+2\beta^2)T\mu + (\frac{K\mu}{2}+\mu^3\delta^2)\sum_{t}\lambda_t^2 \nonumber\\
& = T\kappa K + \frac{\lambda^2+\ln(|\Pi|)}{\mu} + \sum_{t=1}^T \frac{\mu K}{2} (1+(2+\kappa)\lambda_t + \kappa) + (1+\frac{|\beta|}{\delta\mu})\frac{\ln(|\Pi|/\nu)}{\kappa} + (2+2\beta^2)T\mu + (\frac{K\mu}{2}+\mu^3\delta^2)\sum_{t}\lambda_t^2.
\end{align} Note that under the setting of $\delta$ and $\mu$  we have $\frac{\delta\mu}{2}\geq \frac{K\mu}{2} + \mu^3\delta^2$ (we will verify it at the end of the proof), we can drop the terms that relates to $\lambda_t^2$ in the above inequality.  Note that we have $\delta\mu = T^{-\epsilon}\sqrt{K\ln(|\Pi|)}\geq T^{-\epsilon}$, where $\epsilon \in (0,1/2)$. Substitute $\delta\mu \geq T^{-\epsilon}$ into the above inequality and rearrange terms, we get:
\begin{align}
&\sum_{t=1}^T c_t[a_t]  - {w^*}^Ty_t+ \lambda(r_t[a_t] - \beta) - \lambda_t({w^*}^Tz_t - \beta)  -T\delta\mu\lambda^2/2  \nonumber \\
& =  \frac{\lambda^2+\ln(|\Pi|)}{\mu} + T\kappa K + (K+2+2\beta^2 + 2K|\beta|)T\mu + (1 + |\beta|T^{\epsilon})\frac{\ln(|\Pi|/\nu)}{\kappa}
\end{align} Now let us set $\lambda = 0$ and since we have that $\sum_{t=1}^T \lambda_t({w^*}^Tz_t - \beta) \leq 0$, we get:
\begin{align}
&\sum_{t=1} c_t[a_t] - {w^*}^Ty_t \leq \frac{\ln(|\Pi|)}{\mu} + T\kappa K + (K+2+2\beta^2 + 2K|\beta|)T\mu + (1 + |\beta|T^{\epsilon})\frac{\ln(|\Pi|/\nu)}{\kappa} \nonumber \\
&\leq \frac{\ln(|\Pi|)}{\mu} + T\kappa K + (3K+4)T\mu + (1+T^{\epsilon})\frac{\ln(|\Pi|/\nu)}{\kappa} \nonumber \\
& \leq 2 \sqrt{T(\ln(|\Pi|)(3K+4))} +2 \sqrt{T K(1+T^{\epsilon})\ln(|\Pi|/\nu)} = O(\sqrt{T^{1+\epsilon}K\ln(|\Pi|/\nu)})
\end{align} where we set $\mu$ and $\kappa$ as:
\begin{align}
\mu = \sqrt{\frac{\ln(|\Pi|)}{(3K+4)T}}, \;\;\;\;\; \kappa = \sqrt{\frac{(1+T^{\epsilon})\ln(|\Pi|/\nu)}{TK}}.
\end{align}
Now let us consider $\sum_t (r_t[a_t] - \beta)$. Let us assume $\sum_t (r_t[a_t] - \beta) \geq 0$, otherwise we prove the theorem already.  Note that $\sum_{t=1}^T c_t[a_t] - {w^*}^Ty_t \geq -2T$. Hence we have:
\begin{align}
&\lambda \sum_{t=1}^T (r_t[a_t] - \beta) - \lambda^2(\delta\mu T/2 + 1/\mu) \nonumber \\
&\leq 2T +  2 \sqrt{T(\ln(|\Pi|)(3K+4))} +2 \sqrt{T K(1+T^{\epsilon})\ln(|\Pi|/\nu)}. \nonumber
\end{align} To maximize the LHS of the above inequality, we set $\lambda = \frac{\sum_{t=1}^T (r_t[a_t] - \beta)}{\delta\mu T + 2/\mu}$. Substitute $\lambda$ into the above inequality, we get:
\begin{align}
&\big(\sum_{t=1}^T (r_t[a_t] - \beta)\big)^2 \leq (2\delta\mu T+\frac{4}{\mu})(2T +  2 \sqrt{T(\ln(|\Pi|)(3K+4))} +2 \sqrt{T K(1+T^{\epsilon})\ln(|\Pi|/\nu)}) \nonumber \\
& \leq (2T^{1-\epsilon}\sqrt{\ln(|\Pi|)K}+\frac{4}{\mu})(2T +  2 \sqrt{T(\ln(|\Pi|)(3K+4))} +2 \sqrt{T K(1+T^{\epsilon})\ln(|\Pi|/\nu)}) \nonumber\\
& = 24(T^{2-\epsilon}\sqrt{K\ln(|\Pi|)} + T^{1.5-\epsilon}{K\ln(|\Pi|)} + T^{1.5-0.5\epsilon}{K\ln(|\Pi|)} + T^{1.5}\sqrt{K} + TK + T^{1+\epsilon}K\sqrt{\ln(1/\delta)}  \big)\nonumber\\
& = O(T^{2-\epsilon}K\ln(|\Pi|)).
\end{align} 
Hence we have:
\begin{align}
\sum_{t=1}^T (r_t[a_t] - \beta) = O(T^{1-\epsilon/2}\sqrt{K\ln(|\Pi|)}).
\end{align}

Note that for $\delta$, we have $\delta = KT^{-\epsilon+0.5}$. To verify that $\delta \geq \frac{|\beta|}{2/K - \mu - \kappa\mu}$, we can see that as long as $\epsilon \in (0,1/2)$, we have $\delta = \Theta(T^{0.5-\epsilon}) $ while $|\beta|/(2/K - \mu - \kappa\mu) = O(1)$. Hence when $T$ is big enough, we can see that it always holds that $\delta \geq \frac{|\beta|}{2/K - \mu - \kappa\mu}$. For the second condition that $\delta \geq K + 2\mu^2\delta^2 = K + 2\ln(|\Pi|)KT^{-2\epsilon}$. Note that again as long as $\epsilon\in (0, 1/2)$, we have $\delta = \Theta(T^{0.5-\epsilon})$, and  $K+2\ln(|\Pi|)KT^{-2\epsilon} = O(1)$. Hence we have $\delta \geq K + 2\ln(|\Pi|)KT^{-2\epsilon}$. Hence, we have shown that when $\mu = \sqrt{\frac{\ln(|\Pi|)}{(3K+4)T}}$, $\kappa = \sqrt{\frac{(1+T^{\epsilon})\ln(|\Pi|/\nu)}{TK}}$, and $\delta = T^{-\epsilon+1/2}K$, we have that as $T\to\infty$:
\begin{align}
& \sum_{t=1}^T (c_t[a_t] - {w^*}^Ty_t)  = O(\sqrt{T^{1+\epsilon}\ln(|\Pi|/\nu)}), \nonumber\\
&\sum_{t=1}^T (r_t[a_t] -\beta)\leq   O(T^{1-\epsilon/2}\sqrt{K\ln(|\Pi|)}).
\end{align}
\end{myproof}

\end{document}